\documentclass[default,iicol,sn-mathphys]{sn-jnl}

\usepackage{graphicx}%
\usepackage{multirow}%
\usepackage{amsmath,amssymb,amsfonts}%
\DeclareMathOperator*{\argmax}{arg\,max}
\usepackage{amsthm}%
\usepackage{arydshln}
\usepackage{mathrsfs}%
\usepackage[title]{appendix}%
\usepackage{xcolor}%
\usepackage{textcomp}%
\usepackage{manyfoot}%
\usepackage{booktabs}%
\usepackage{algorithm}%
\usepackage{array}
\usepackage{algorithmicx}%
\usepackage{algpseudocode}%
\usepackage{listings}%
\usepackage{pifont}  
\newcommand{\bcmark}{\ding{51}} %
\newcommand{\bxmark}{\ding{55}} %
\newtheorem{corollary}{Corollary}
\newtheorem{lemma}{Lemma}

\theoremstyle{thmstyleone}%
\newtheorem{theorem}{Theorem}
\theoremstyle{thmstyletwo}%

\theoremstyle{thmstylethree}%

\begin{document}

\title[Article Title]{An Information Theory-inspired Strategy for Automat{ed} Network Pruning}


\author[1]{\fnm{Xiawu} \sur{Zheng}}

\author[1]{\fnm{Yuexiao} \sur{Ma}}

\author[2]{\fnm{Teng} \sur{Xi}}

\author[2]{\fnm{Gang} \sur{Zhang}}

\author[2]{\fnm{Errui} \sur{Ding}}

\author[1]{\fnm{Yuchao} \sur{Li}}

\author[3]{\fnm{Jie} \sur{Chen}}

\author[4]{\fnm{Yonghong} \sur{Tian}}

\author*[1]{\fnm{Rongrong} \sur{Ji}}

\affil[1]{\orgdiv{Media Analytics and Computing Lab}, \orgname{Department of Artificial Intelligence, School of Informatics}, \orgaddress{\city{Xiamen}, \country{China}}}

\affil[2]{\orgdiv{Department of Computer Vision Technology (VIS)}, \orgname{Baidu Inc},  \orgaddress{\city{Beijing}, \country{China}}}

\affil[3]{\orgdiv{Institute of Digital Media} \orgname{Peking University},  \orgaddress{\city{Beijing}, \country{China}}}

\affil[4]{\orgdiv{National Engineering Laboratory for Video Technology (NELVT)}, \orgname{School of Electronics Engineering and Computer Science}, \orgaddress{\city{Beijing}, \country{China}}}


\abstract{{Despite superior performance achieved on many computer vision tasks, deep neural networks demand high computing power and memory footprint.} 
Most existing network pruning methods require laborious human efforts and prohibitive computation resources, especially when the constraints are changed. 
This practically limits the application of model compression when the model needs to be deployed on a wide range of devices. 
Besides, existing methods are still challenged by the missing theoretical guidance, \textcolor{black}{which lacks influence on the generalization error}. 
In this paper we propose an information theory-inspired strategy for {automated} {network pruning}. 
The principle behind our method is the information bottleneck theory. 
Concretely, {we introduce a new theorem to illustrate that} the hidden representation should compress information with each other {to achieve a better generalization}. 
{In this way,} we further introduce the normalized Hilbert-Schmidt Independence Criterion (nHSIC) on network activations as a stable and generalized indicator {to construct} layer importance.
When a certain resource constraint is given, we integrate the HSIC indicator with the constraint to transform the architecture search problem into a linear programming problem with quadratic constraints. 
Such a problem is easily solved by a convex optimization method {within} a few seconds. 
We also provide rigorous proof to reveal that optimizing the normalized HSIC simultaneously minimizes the mutual information between different layers. 
Without any search process, our method achieves better compression trade-offs {compared} to the state-of-the-art compression algorithms. 
For instance, {on} ResNet-$50$, we achieve a $45.3\%$-FLOPs reduction, with a $75.75$ top-1 accuracy on ImageNet. 
Codes are available at \url{https://github.com/MAC-AutoML/ITPruner}.}




\maketitle

\section{Introduction}\label{sec:intro}

While convolutional neural networks (CNNs) \cite{lecun2015deep} have demonstrated great success in various computer vision tasks, such as classification \cite{he2016deep, krizhevsky2012imagenet}, detection \cite{Ren2016fastrcnn, girshick2015fast} and semantic segmentation \cite{chendeeplab, long2015fully}, their large demands of computation power and memory footprint make most state-of-the-art CNNs notoriously challenging to deploy on resource-limited devices such as smartphones or wearable devices. 
Lots of research resources have been devoted to CNN compression and acceleration, including but not limited to parameter quantization \cite{jacob2017quantization, zhou2016dorefa}, filter compression \cite{wen2016learning, luo2017thinet, li2016pruning} and {automated} network pruning \cite{yang2018netadapt, wang2020revisiting, dong2019network, liu2019metapruning, yu2019autoslim, he2018amc}.

Among these compression methods, {automated} network pruning has recently attracted a lot of attentions and has been successfully applied to state-of-the-art architectures, such as EfficientNet \cite{tan2019efficientnet}, MnasNet \cite{tan2019mnasnet} and MobileNetV3 \cite{howard2019searching}. 
Given a pre-trained network, the compression method aims to automatically {reduce} the channel of the network until the resource budget is met while maximizing the accuracy. 
It is also named as channel number search \cite{dong2019network, yu2019autoslim, wang2020revisiting}, {automated} model compression \cite{he2018amc} or network adaption \cite{yang2018netadapt}. 
Based on the design of the methodology, we empirically categorize network pruning methods into two groups and discuss them below.

\textbf{Metric\textcolor{black}{-}based methods} compress the size of a network by selecting filters \textcolor{black}{that} rely on hand-crafted metrics. 
We can further divide existing metrics into two types, local metrics \cite{li2016pruning, luo2017thinet, wen2016learning} and global metrics \cite{lin2018accelerating, molchanov2019taylor}. 
The first pursues to identify the importance of filters inside a layer. 
In other words, local metric\textcolor{black}{-}based methods require human experts to design and decide hyper-parameters like channel number and then pruning filters in each layer according to $L1$ norm \cite{li2016pruning} or geometric median \cite{he2019filter}. These methods are not {automated} and thus less practical in compressing various models. 
In addition, extensive experiments in {previous works} \cite{blalock2020state, liu2018rethinking} suggest that filter-level pruning does not help as much as selecting a better channel number for {an} architecture.
The second designs a metric to compare the importance of filters across different layers. 
Such methods implicitly decide the channel number and thus alleviate \textcolor{black}{a} large amount of human effort compared to local metric\textcolor{black}{-}based methods. 
However, these methods usually perform network pruning followed by a data-driven and/or iterative optimization to recover accuracy, both of which are time-cost.

\textbf{Search\textcolor{black}{-}based methods.}
Apart from these human-designed heuristics, efforts on search\textcolor{black}{-}based methods have been made recently. 
The main difference between these methods lies in the search algorithm and architecture evaluation method. 
In {terms} of \textcolor{black}{the} search algorithm, AMC \cite{he2018amc} and NetAdapt \cite{yang2018netadapt} leverage reinforcement learning to efficiently sample the search space. 
Autoslim \cite{yu2019autoslim} greedily slim the layer with minimal accuracy drop. 
In terms of the evaluation method, existing methods usually adopt {a} weight sharing strategy to evaluate the accuracy \cite{wang2020revisiting, dong2019network, liu2019metapruning, yu2019autoslim, he2018amc}. 
To explain, weight sharing methods maintain an over-parameterized network that covers the entire search space. 
The sampled architectures directly inherit the corresponding weights from the over-parameterized network, which is then used to evaluate the performance and update the weights. 
{Apart from these methods, \emph{Yuan et al.} \cite{yuan2020growing} propose an efficient network growing method by using structured continuous sparsification, which produces a compact neural network and drastically reduces the computational expense of training.}
Both of these methods require recompression or retraining while the constraints {change}. 
This practically limits the application of network pruning since the model needs to be deployed on a wide range of devices.

In this paper, we propose an information theory-inspired pruning (ITPruner) strategy that does not need the aforementioned iterative training and search process, which is simple and straightforward. 
Specifically, we first introduce the normalized Hilbert-Schmidt Independence Criterion (nHSIC) on network activations as an accurate and robust layer-wise importance indicator. 
Such importance is then combined with the constraint to convert the architecture search problem to a linear programming problem with bounded variables. 
In this way, we obtain the optimal architecture by solving the linear programming problem, which only takes a few seconds on a single CPU and GPU. 

Our method is motivated by the information bottleneck (IB) theory \cite{tishby99information}. 
That is, for each network layer, we {\textcolor{black}{theoretically} prove that one} should minimize the information between the layer activation and input {for a better generalization}. 
We generalize such a principle to different layers for {automated} network pruning, \emph{i.e., a layer correlated to other layers is less important.} 
In other words, we build a connection between network redundancy and information theory. 
However, calculating the mutual information among the intractable distribution of the layer activation is impractical. 
We thus adopt a non-parametric kernel-based method the Hilbert-Schmidt Independence Criterion (HSIC) to characterize the statistical independence. 
In summary, our contributions are as follows:
\begin{itemize}
	\item \textbf{Information theory on deep learning redundancy.} To our best knowledge, this is the first to build a relationship between {mutual information} and \textcolor{black}{the} redundancy of CNN by generalizing the information bottleneck principle. {Specifically, we provide rigorous proof to show that minimizing layer-wise mutual information will lead to a better generalization.} We are also the first that \textcolor{black}{applies} such {a} relationship in {automated} network pruning. \footnote{Previous work \cite{dai2018compressing} also proposed to compress the network using \textcolor{black}{the} variational information bottleneck. However, their metric is only applied inside a specific layer, which is not {automated} and less effective.}
	\item \textbf{Methodology.} We propose a unified framework to automatically compress a network without any search process. The framework integrates HSIC importance with constraints to transform the network pruning to the convex optimization problem, which is efficient for deployment on various devices.
	\item \textbf{Theoretical contributions.} Except for the theoretical guarantees in {information bottleneck}, we also theoretically prove the robustness of HSIC as well as the relationship between our method and mutual information. In particular, as proved in Sec.~\ref{sec:theo}, optimizing HSIC is equivalent to {minimizing} the layer-wise mutual information under the gaussian modeling.
\end{itemize}

The experimental results demonstrate the efficiency and effectiveness on different network architectures and datasets. 
Notably, without any search process, our compressed MobileNetV1 obtain $50.3\%$ FLOPs reduction with $70.92$ Top-$1$ accuracy. 
Meanwhile, the compressed MobileNet V$2$, ResNet-$50$ and MobileNetV$1$ obtained by ITPruner achieve the best performance gains over the state-of-the-art AutoML methods.  

\begin{figure*}[t]
	\centering
	\includegraphics[width=1.0\linewidth]{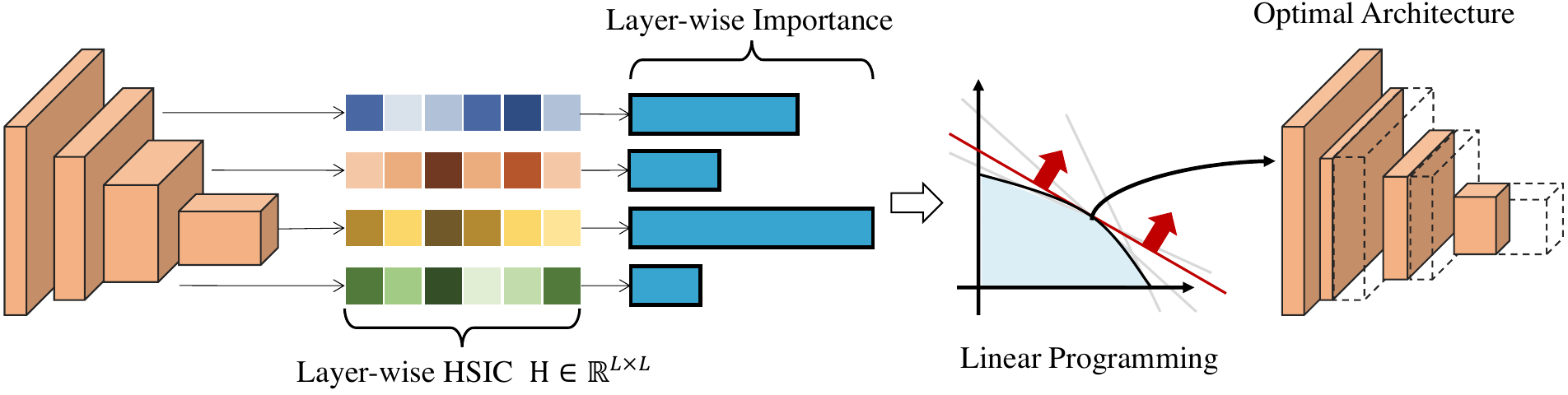}
	\caption{Overview of our ITPruner. Specifically, We first sample $n$ images to obtain the feature map of each layer. Then the normalized HSIC is employed to calculate the independence map $\boldsymbol{H}$ between different layers. For each layer, we sum the elements of the corresponding column in $\boldsymbol{H}$ except for itself as the importance indicator of the layer. In this way, we model the layer-wise importance and compression constraints as \textcolor{black}{the} linear programming problem. Finally, the optimal network architecture is obtained by solving the optimal solution on the linear programming.}
	\label{fig:framework}
\end{figure*}

\section{Related Work}\label{sec2}

There are a large amount {of} works that aim at compressing CNNs and employing information theory {in} deep learning. {Most network compression methods are introduced in Sec.~\ref{sec:intro}, where the rest} approaches are summarized as follows.

\textbf{{Automated} Network Pruning.} Following previous works \cite{yang2018netadapt, wang2020revisiting, dong2019network, liu2019metapruning, yu2019autoslim, he2018amc}, {automated} network pruning adapts channels of a pre-trained model automatically to a specific platform (\emph{e.g.,} mobile, IoT devices) under a given resource budget. 
Similar works that {modify} channels to reduce its FLOPs and speedup its inference have been proposed \cite{wen2016learning, luo2017thinet, li2016pruning, wang2020revisiting, dong2019network, liu2019metapruning, yu2019autoslim, he2018amc} in the literature, which is known as network pruning. 
Most of the pruning-based methods have been discussed in Sec.~\ref{sec:intro}. 
Therefore, we select the two most related works \cite{molchanov2019taylor, lin2018accelerating}, which proposed to prune networks by designing a global metric. 
However, these methods are still far from satisfactory. 
On one hand, there is a significant performance and architecture gap compared to search\textcolor{black}{-}based methods \cite{liu2019metapruning, yu2019autoslim, he2018amc}. 
On the other hand, these methods also need an extra iterative optimization step to recover the performance. 
In this case, search\textcolor{black}{-}based methods and global metric-based methods require almost the same computation resources. 
In contrast, our method shows superior performance compared to search based methods. 
Besides, we do not need any iterative training or search in our method, and generating an optimal architecture only takes a few seconds.

\textbf{ Information theory}\cite{cover1999elements} is a fundamental research area which often measures \textcolor{black}{the} information of random variables associated with distributions. 
Therefore, a lot of principles and theories have been employed to explore the learning dynamic \cite{goldfeld2018estimating, alemi2016deep} or as a training objective to optimize the network without stochastic gradient descent \cite{ma2020hsic, shwartz2017opening}. 
To our best knowledge, there exist only two prior works \cite{dai2018compressing, ahn2019variational} that apply information theory in network compression. 
Dai \textit{et al.} proposed to compress the network using variation information bottleneck \cite{tishby99information}. However, their metric measures importance inside a specific layer, which is not {automated} and less effective. 
Ahn \textit{et al.} proposed a variational information distillation method, aiming to distill the knowledge from the teacher to the student. 
Such a method is orthogonal to our method. In other words, we can employ the method as a post-processing step to further improve the performance. 

\section{Methodology}\label{sec3}

\noindent \textbf{Notations.} Upper case (\emph{e.g.,} $X, Y$) denotes random variables. Bold denotes vectors (\emph{e.g.,} $\boldsymbol{x, y}$), matrices (\emph{e.g.,} $\boldsymbol{X, Y}$) or tensors. Calligraphic font denotes spaces (\emph{e.g.,} $\mathcal{X}, \mathcal{Y}$). We further define some common notations in CNNs. The $l$-th convolutional layer in a specific CNN convert\textcolor{black}{s} an input tensor $\boldsymbol{X}^l \in \mathbb{R}^{c^l \times h_{\text{in}}^l \times w_{\text{in}}^l}$ to an output tensor $\boldsymbol{Y}^l \in \mathbb{R}^{n^l \times h_{\text{out}}^l \times w_{\text{out}}^l}$ by using a weight tensor $\boldsymbol{W}^l \in \mathbb{R}^{n^l \times c^l \times k^l \times k^l}$, where $c^l$ and $n^l$ denote the numbers of input channels and filters (output channels), respectively, and $k^l\times k^l$ is the spatial size of the filters.

\noindent \textbf{Problem formulation.} Given a pre-trained CNN model that contains $L$ convolutional layers, we refer to $\boldsymbol{\alpha} = (\alpha_1, \alpha_2, ... , \alpha_L)$ as the desired architecture, where $\alpha_l$ is the compression ratio of $l$-th layers. Formally, we address the following optimization problem:
\begin{equation}\label{eq:prob}
	\max_{\boldsymbol{W} \in \mathcal{W}, \boldsymbol{\alpha} \in \mathcal{A}} f(\boldsymbol{W} , \boldsymbol{\alpha}), \, s.t. \; T(\boldsymbol{\alpha}) < {\Omega},
\end{equation}
where $f:\mathcal{W} \times \mathcal{A} \rightarrow \mathbb{R}$ is the objective function that is differentiable \emph{w.r.t}. the network weight $\boldsymbol{W} \in \mathcal{W}$, and is not differentiable \emph{w.r.t.} the architecture indicator vector $\boldsymbol{\alpha} \in \mathcal{A} \in\mathbb{R}^{L \times 1}$. 
$T(.)$ is a function that denotes the constraints for architectures, such as FLOPs and latency\footnote{For example, most mobile devices have the constraint that the FLOPs should be less than $600$M.}, and $\Omega$ is given for different hardwares.

We present an overview of ITPruner in Fig.~\ref{fig:framework}, which aims to automatically discover the redundancy of each layer in a pre-trained network. 
Notably, the redundancy is also characterized by the layer-wise pruning ratio or sparsity. 
The detailed motivations, descriptions and analysis are presented in the following sub-sections.

\subsection{Information Bottleneck in Deep Learning}\label{sec:info_bo}

The information bottleneck (IB) \cite{tishby99information} principle express\textcolor{black}{es} a tradeoff of mutual information from the hidden representation to input and output, which is formally defined as 
\begin{equation}\label{eq:basic_IB}
	\min_{p(X_i|X)} I(X;X_i)-\beta I(X_i;Y),
\end{equation}
where $X,Y$ are the random variables of input and label, and $X_i$ represents the hidden representation. 
Intuitively, the IB principle is proposed to compress the mutual information between the input and the hidden representation while preserving the information between the hidden representation about the input data. 
Here we introduce a new theorem that provide\textcolor{black}{s} an evidence \textcolor{black}{for} the IB \textcolor{black}{principle} in deep learning. 
\begin{theorem} \label{theorem:compression_bottle}
    Assuming $X$ is the input random variable follows a Markov random field structure and the Markov random field is ergodic. 
    For a network that have $L$ hidden representations $X_1, X_2, ..., X_L$, with a probability $1-\delta$, the generalization error $\epsilon$ is bounded by
    \begin{equation}\nonumber
        \epsilon \leq \sum_{i=1}^L\sqrt{\frac{\log \frac{2}{\delta} + \log2\text{I}(X ;X_i)}{2n}},
    \end{equation}
    where $n$ is the number of training examples. 
\end{theorem}
\noindent {Theorem \ref{theorem:compression_bottle} indicates that the bound of the generalization error increases as the layer-wise mutual information grows.} 
\textcolor{black}{In other words, it provides a mathematical methodology to directly influence the generalizability of deep neural networks, which is essential in network pruning. As network pruning always wants to compress a network work with high performance in the validation set.} 
Considering that {hidden representations in CNNs are the input of the subsequent layers,} we further {reformulate Theorem \ref{theorem:compression_bottle}} as
\begin{equation}\label{eq:general_IB_previous}
	\min_{p(X_i|X)} \sum_{i=1}^L \sum_{j=i+1}^L \left(I(X;X_i) +  I(X_j;X_i)  \right) -\beta I(X_i;Y).
\end{equation}
Note that there \textcolor{black}{is} always a fine-tuning process in the pruning process. We thus further simplify the IB principle as
\begin{equation}\label{eq:general_IB}
	\min_{p(X_i|X)} \sum_{i=1}^L \sum_{j=i+1}^L \left(I(X;X_i) +  I(X_j;X_i)  \right).
\end{equation}
According to the generalized IB principle, we hope that the mutual information between different layers is close to $0$. 
In other words, representations from different layers are best to be independent of each other. 
We thus generalize such a principle in network compression, \emph{i.e., a layer correlated to other layers is less important.} 
Meanwhile, layer-wise mutual information can be considered as a robust and accurate indicator to achieve search-free {automated} network pruning. 
\textcolor{black}{Another possible way to reduce generalization error is MKR \cite{yang2021multiple}, which reduces data bias by using symbolic knowledge. In other words, we can also employ ITPruner with MKR when the data bias is accessible.}

The IB is hard to compute in practice. On one hand, the distribution of the hidden representation is intractable in CNN. Previous work \cite{ahn2019variational} defines a variational distribution that approximates the distribution. 
However, such a method also introduces a new noise between the approximated and true distribution, which is inevitable as the Kullback-Leiber divergence is non-negative. 
On the other hand, the hidden representations usually have a high dimension in CNN. 
In this case, many algorithms based on binning {suffer} from the curse of dimensionality. 
Besides, they also yield different results with different bin sizes.

\subsection{Normalized HSIC Indicator}

To solve the issues in Sec.~\ref{sec:info_bo}, we introduce the normalized HSIC to replace the mutual information term in Eq.~\ref{eq:general_IB}. 
Formally, the Hilbert-Schmidt Independence Criterion (HSIC) \cite{gretton2005measuring} is defined as
\begin{equation}
	\begin{aligned}
	\text{HSIC}( & P_{XY}, \mathcal{H}, \mathcal{G}) = ||C_{XY}||^2	\\
	& = \mathbb{E}_{X,X^{'},Y,Y^{'}}[k_X(X, X^{'})k_Y(Y, Y^{'})] \\
	& + \mathbb{E}_{X,X^{'}}[k_X(X, X^{'})]\mathbb{E}_{Y,Y^{'}}[k_Y(Y, Y^{'})] \\
	& - 2\mathbb{E}_{XY}[\mathbb{E}_{X^{'}}[k_X(X, X^{'})] \mathbb{E}_{Y^{'}}[k_Y(Y, Y^{'})]].
	\end{aligned}
\end{equation}
Here $k_X, k_Y$ are kernel functions, $\mathcal{H}, \mathcal{G}$ are the Reproducing Kernel Hilbert Space (RKHS) and $\mathbb{E}_{X,X^{'},Y,Y^{'}}$ denotes the expectation over independent examples $(x,y), (x^{'}, y^{'})$ drawn from $p_{xy}$. 
In order to make HSIC practically available, Gretton \emph{et al.} \cite{gretton2005measuring} further proposed to empirically approximate $\text{HSIC}(P_{XY}, \mathcal{H}, \mathcal{G})$ given a finite number of observations. 
Specifically, let $\mathcal{D}:=\{(x_1, y_1), ...,(x_n, y_n)\}$ contain $n$ i.i.d samples drawn from  the distribution $P_{XY}$. 
The estimation of HSIC is given by
\begin{equation}\label{eq:esti_HSIC}
	\text{HSIC}(\mathcal{D}, \mathcal{F}, \mathcal{G}) = (n-1)^{-2}\text{tr}(\boldsymbol{K}_X\boldsymbol{H}\boldsymbol{K}_Y\boldsymbol{H}),
\end{equation}
where $\boldsymbol{K}_X, \boldsymbol{K}_Y, \boldsymbol{H} \in \mathbb{R}^{n\times n}$, $\boldsymbol{K}_X $ and $ \boldsymbol{K}_Y$ have entries $\boldsymbol{K}_{X(i,j)} = k(x_i, x_j)$, and $\boldsymbol{H} = \boldsymbol{I}_n - \frac{1}{n}\boldsymbol{1}_n \boldsymbol{1}^T_n $ is the centering matrix. 
Notably, an alternative implementation of $\boldsymbol{H}$ is to centralize the examples, \emph{i.e.,} $x_i = x_i - \mu_X$, where $\mu_X = \frac{1}{n} \sum_i X_i$. 
In this way, Eq.~\ref{eq:esti_HSIC} is further simplified as
\begin{equation}\label{eq:simp_HSIC}
	\text{HSIC}(X,Y) = (n-1)^{-2}\text{tr}(\boldsymbol{K}_X\boldsymbol{K}_Y),
\end{equation}
where $\boldsymbol{K}_X,\boldsymbol{K}_Y$ are centralized matrices. In our paper, we use the normalized HSIC (nHSIC) based on the centered kernel alignment \cite{kornblith2019similarity}, given by
\begin{equation}\label{eq:norm_HSIC}
	\text{nHSIC}(X, Y) = \frac{\text{tr}(\boldsymbol{K}_X \boldsymbol{K}_Y )}{\sqrt{\text{tr}(\boldsymbol{K}_X \boldsymbol{K}_X)}\sqrt{\text{tr}(\boldsymbol{K}_Y \boldsymbol{K}_Y)}}
\end{equation}

The HSIC can be effectively computed in $\mathcal{O}(n^2)$ time. Meanwhile, the uniform convergence bounds derived in \cite{gretton2005measuring} with respect to $P_{XY}$ is $\mathcal{O}(n^{-\frac{1}{2}})$. 
When an appropriate kernel is selected, $\text{HSIC} = 0$ if and only if the random variables $X, Y$ are independent, \emph{i.e.,} $P_{XY} = P_X P_Y $. 
We also provide proof in Sec.~\ref{sec:theo} to demonstrate the relationship between HSIC and mutual information: When the linear kernel is selected, minimizing HSIC is equivalent to {minimizing} the mutual information under a {gaussian} assumption. 
Meanwhile, in this case, nHSIC is also equivalent to the RV coefficient \cite{robert1976unifying} and Tucker's congruence coefficient \cite{lorenzo2006tucker}.

\subsection{Pruning Strategy}

\textbf{{Obtain $\boldsymbol{\alpha}$}}. We now describe how to use the nHSIC to {obtain $\boldsymbol{\alpha}$} under an optimal trade-off between compression rate and accuracy. 
Specifically, we first sample $n$ images to obtain the feature map of each layer $\boldsymbol{X}^1, \boldsymbol{X}^2, ..., \boldsymbol{X}^L$ for the network that needs to be compressed. 
We then use these feature maps and Eq.~\ref{eq:norm_HSIC} to get the independence between different layers, thereby constructing an independence matrix $\boldsymbol{H} \in \mathbb{R}^{L\times L}$, where $\boldsymbol{H}$ has entries $\boldsymbol{H}_{i,j} = \text{nHSIC}(\boldsymbol{X^i}, \boldsymbol{X^j})$. 
As mentioned before, a layer correlated to other layers is less important. 
Therefore, the importance of a specific layer $l$ is formally defined as
\begin{equation}\label{eq:import}
	\boldsymbol{i}^l = e^{-\beta \sum_{i=1, i\neq l}^L \text{nHSIC}(\boldsymbol{X}^l, \boldsymbol{X}^i ) },
\end{equation}
where $\beta$ is proposed to control \textcolor{black}{the} relative compression rate between different layers. 
To explain, the compressed network \textcolor{black}{tends} to \textcolor{black}{have} a significant compression rate difference between different layers with a small $\beta$, and vice versa. 
We also provide extensive experiments in Sec.~\ref{sec:expri} to verify the conclusion. 
With the layer-wise importance $\boldsymbol{i} \in \mathbb{R}^{L\times 1}$, the problem of Eq.~\ref{eq:prob} is transformed to a linear programming problem that is differentiable \emph{w.r.t} $\alpha$. 
Formally, Eq.~\ref{eq:prob} is rewritten as
\begin{equation}\label{eq:HIS_prob}
	\max_{\boldsymbol{\alpha}} \, \boldsymbol{i}^T\boldsymbol{\alpha} \, s.t. \; T(\boldsymbol{\alpha}) < {\Omega}.
\end{equation}
In most cases, the constraints such as FLOPs, parameters, and inference time can be decomposed to the sum over different layers. 
Therefore, Eq.~\ref{eq:HIS_prob} is further simplified to a linear programming \textcolor{black}{problem} with quadratic constraints
\begin{equation}\label{eq:final_prob}
	\max_{\boldsymbol{\alpha}} \, \boldsymbol{i}^T\boldsymbol{\alpha} \, s.t. \; \boldsymbol{\alpha}^T\boldsymbol{T} \boldsymbol{\alpha} < {\Omega},
\end{equation}
where $\boldsymbol{T} \in \mathbb{R}^{L\times L}$ is the constraint factor matrix corresponding to the layers. 
{Notably, $\boldsymbol{T}$ is compatible with different constraints, including but not limited to model size, FLOPs, and Inference time\footnote{{FBnet~\cite{wu2019fbnet} proposed a method using a latency table to predict the latency on specific hardware. In ITPruner, we can measure the inference layer by layer to obtain an estimator using $\boldsymbol{\alpha}$.} }.}
For example, suppose we choose the {FLOPs}~\cite{MolchanovICLRPruning} as the constraint. 
{According to the defination, the constraint in Eq.~\ref{eq:final_prob} is further specified to
\begin{equation}\nonumber
\begin{aligned}
    \boldsymbol{\alpha}^T\boldsymbol{T} \boldsymbol{\alpha} = & 2h_{\text{out}}^1w_{\text{out}}^1\left(3k^1k^1+1\right)\alpha^1n^1 + \\ 
    &\sum_{l=2}^L 2h_{\text{out}}^lw_{\text{out}}^l\left(\alpha^{l-1}c^lk^lk^l+1\right)\alpha^ln^l.
\end{aligned}
\end{equation}}
{In other words,} the constraints are quadratic in Eq.~\ref{eq:final_prob}.
{Meanwhile}, solving the problem in Eq.~\ref{eq:final_prob} is extremely effective, which only takes a few seconds on a single CPU by using the solver in \cite{kraft1988software}. 
Our proposed ITPruner is also summarized in Alg.~\ref{alg:ITPruner}. 

{\textbf{Network Pruning.}
After $\boldsymbol{\alpha}$ is obtained, we then pruning and finetune the network to compress model size and retain the accuracy, respectively. As mentioned in \cite{blalock2020state, liu2018rethinking}, selecting which weights in a layer need to be pruned does not help as much as finding a better $\boldsymbol{\alpha}$. Meanwhile, our method is orthogonal with different local metric based pruning methods, which can be incorporated to achieve better performance. For simplicity, $L1$ norm \cite{li2016pruning} is selected as the pruning criterion. Specifically, to pruning $l-$th layer, we first convert the $4-$D weight tensor $\boldsymbol{W}^l \in \mathbb{R}^{n^l \times c^l \times k^l \times k^l}$ to the $2-$D matrix $\boldsymbol{W}^{l'} \in \mathbb{R}^{n^l \times \left[ c^l \times k^l \times k^l\right]}$. In this case, the weight matrix $\boldsymbol{W}^{l'}$ is considered as a collection of $n^l$ 3$-D$ filters. For each filters, we then calculate the sum of its absolute as $\boldsymbol{s}^l \in \mathbb{R}^{n^l \times 1}, \boldsymbol{s}^l_i = \sum_{j=1}^{\left[ c^l \times k^l \times k^l\right]} |\boldsymbol{W}^{l'}_{i,j}|$. After that, $\lfloor \boldsymbol{\alpha}^ln^l \rfloor$ filters with the smallest values and their corresponding feature maps are pruned in the layer. The feature maps and the filters in the next convolutional layer are removed accordingly. Finally, we train the pruned network to recover the accuracy. It is worth noting that ITPruner is orthogonal to traditional pruning methods, including structured pruning and unstructured pruning. Therefore, we also integrate unstructured pruning in our method, where $\boldsymbol{W}^l$ is directly converted to a vector and $\lfloor \boldsymbol{\alpha}^ln^l \rfloor$ elements with the smallest $L1$ value are pruned for compression.
}


\renewcommand{\algorithmicrequire}{\textbf{Input:}}
\renewcommand{\algorithmicensure}{\textbf{Output:}}

\begin{algorithm}
\caption{ITPruner\label{alg:ITPruner}}
\begin{algorithmic}[1]
\Require Sampled $n$ i.i.d samples $\mathcal{D}$, pre-trained model $\mathcal{M}$.
\Ensure Optimized architecture $\boldsymbol{\alpha}^*$.
\State Input $\mathcal{D}$ into $\mathcal{M}$ to generate layer-wise feature map $\boldsymbol{X}^1, ... , \boldsymbol{X}^L$;\\
Obtain the independent matrix $\boldsymbol{H}$ by using Eq.~\ref{eq:norm_HSIC};\\
Calculate the layer-wise importance $\boldsymbol{i}$ by Eq.~\ref{eq:import};\\
$\alpha^* = \argmax_{\boldsymbol{\alpha}} \, \boldsymbol{i}^T\boldsymbol{\alpha} \, s.t. \; \boldsymbol{\alpha}^T\boldsymbol{T} \boldsymbol{\alpha} < {\Omega}$ ;
\end{algorithmic}
\end{algorithm}

\textbf{Discussion.} Compared to previous methods \cite{wang2020revisiting, dong2019network, liu2019metapruning, yu2019autoslim, he2018amc}, ITPruner is effective and easy to use. 
In terms of effectiveness, our method does not need \textit{any training and search} process to obtain the optimal architecture. 
The time complexity is only correlated to the sample number $n$, \emph{i.e.,} $\mathcal{O}(n)$ in the feature generation and $\mathcal{O}(n^2)$ in HSIC calculation, both of them only take a few seconds on a single GPU and CPU, respectively. 
In terms of usability, there are only two hyper-parameters $n$ and $\beta$ in our method. 
Meanwhile, $n$ and $\beta$ are generalized to different architectures and datasets. 
Extensive experiments in Sec.~\ref{subsec:abla} demonstrate that the performance gap between different $\beta$ and $n$ is negligible.
This is in stark contrast to the previous search-based methods, as they need to exhaustedly select a large \textcolor{black}{number} of hyper-parameters for different architectures.

\section{Theoretical Analysis}\label{sec:theo}

In this section, we provide theoretical analysis about the proposed method. Specifically, the invariance of scale and orthogonal transformation is already proposed in \cite{kornblith2019similarity}, which is also the key property of feature similarity metric \cite{kornblith2019similarity} in deep learning. Here we provide the strict proofs about the {aforementioned} invariance in Theorem \ref{theo:scale_invar} and Theorem \ref{theo:orth_invar} respectively. To further explain the relationship between $\text{nHSIC}_{\text{linear}}$ and mutual information, we theoretically prove that optimizing $\text{nHSIC}_{\text{linear}}$ is equivalent to minimizing the mutual information in Theorem \ref{theo:hsic_MI}.


To make the implementation easy, we use linear kernel $k(\boldsymbol{x}, \boldsymbol{y}) = \boldsymbol{x}^T \boldsymbol{y}$ in all our experiments. 
In this case, $ \boldsymbol{K}_X=XX^T, \boldsymbol{K}_Y=YY^T $. 
And nHSIC is further simplified as 
\begin{equation}\label{eq:linear_HSIC}
\small
\begin{aligned}
	&\text{nHSIC}_{\text{linear}}(X, Y) \\
	&= \frac{\text{tr}\left(\boldsymbol{XX}^T \boldsymbol{YY}^T \right)}{\sqrt{\text{tr}\left(\boldsymbol{XX}^T \boldsymbol{XX}^T\right)}\sqrt{\text{tr}\left(\boldsymbol{YY}^T \boldsymbol{YY}^T\right)}}\\
	&= \frac{\left\langle \text{vec}\left(\boldsymbol{XX}^T\right),\text{vec}\left(\boldsymbol{YY}^T\right) \right\rangle}{\sqrt{\left\langle \text{vec}\left(\boldsymbol{XX}^T\right),\text{vec}\left(\boldsymbol{XX}^T\right) \right\rangle}\sqrt{\left\langle \text{vec}\left(\boldsymbol{YY}^T\right),\text{vec}\left(\boldsymbol{YY}^T\right) \right\rangle}}\\
	&=\frac{||\boldsymbol{Y}^T \boldsymbol{X} ||_F^2}{||\boldsymbol{X}^T \boldsymbol{X} ||_F ||\boldsymbol{Y}^T \boldsymbol{Y} ||_F}, 
\end{aligned}
\end{equation}
where $||.||_F$ is Frobenius norm or the Hilbert-Schmidt norm. 
Eq.~\ref{eq:linear_HSIC} is easy to compute and has some useful properties of being invariant to scale and orthogonal transformations. 
To explain, we provide two formal theorems and the corresponding proofs as follows
\begin{theorem}\label{theo:scale_invar}
$\text{nHSIC}_{\text{linear}}$ is invariant to scale transformation, i.e., $\text{nHSIC}_{\text{linear}}(\beta X, Y) = \text{nHSIC}_{\text{linear}}(X, Y)$.
\end{theorem}

\begin{proof}
	\begin{equation}\nonumber
	\begin{aligned}
		\text{nHSIC}_{\text{linear}}(\beta X, Y) &= 
		\frac{\beta^2||\boldsymbol{Y}^T \boldsymbol{X} ||_F^2}{\beta \cdot||\boldsymbol{X}^T \boldsymbol{X} ||_F \beta \cdot ||\boldsymbol{Y}^T \boldsymbol{Y} ||_F}\\
		&=\text{nHSIC}_{\text{linear}}(X, Y).
	\end{aligned}
	\end{equation}
\end{proof}
\begin{theorem}\label{theo:orth_invar}
	$\text{nHSIC}_{\text{linear}}$ is invariant to orthogonal transformation $\text{nHSIC}_{\text{linear}}(X \boldsymbol{U}, Y) = \text{nHSIC}_{\text{linear}}(X, Y)$, where $\boldsymbol{U} \boldsymbol{U}^T = \boldsymbol{I}$.
\end{theorem}
\begin{proof}
\begin{equation}\nonumber
	\begin{aligned}
		&\text{nHSIC}_{\text{linear}}( X\boldsymbol{U}, Y) \\
		&= 
		\frac{||\boldsymbol{Y}^T \boldsymbol{X} \boldsymbol{U} ||_F^2}{||\boldsymbol{U}^T \boldsymbol{X}^T \boldsymbol{X} \boldsymbol{U} ||_F ||\boldsymbol{Y}^T \boldsymbol{Y} ||_F}\\
		&=\frac{tr(\boldsymbol{X} \boldsymbol{U} \boldsymbol{U}^T \boldsymbol{X}^T \boldsymbol{Y} \boldsymbol{Y}^T)}{\sqrt{tr(\boldsymbol{X} \boldsymbol{U} \boldsymbol{U}^T \boldsymbol{X}^T \boldsymbol{X} \boldsymbol{U} \boldsymbol{U}^T \boldsymbol{X}^T)} ||\boldsymbol{Y}^T \boldsymbol{Y} ||_F}\\
		&=\frac{tr(\boldsymbol{X} \boldsymbol{X}^T  \boldsymbol{Y} \boldsymbol{Y}^T)}{\sqrt{tr(\boldsymbol{X} \boldsymbol{X}^T \boldsymbol{X} \boldsymbol{X}^T)} ||\boldsymbol{Y}^T \boldsymbol{Y} ||_F}\\
		&=\frac{||\boldsymbol{Y}^T \boldsymbol{X} ||_F^2}{||\boldsymbol{X}^T \boldsymbol{X} ||_F ||\boldsymbol{Y}^T \boldsymbol{Y} ||_F}\\
		&=\text{nHSIC}_{\text{linear}}(X, Y),
	\end{aligned}
\end{equation}
which means that $\text{nHSIC}_{\text{linear}}$ is invariant to orthogonal transformation. 
\end{proof}
Theorem \ref{theo:scale_invar} and \ref{theo:orth_invar} illustrate the robustness of $\text{nHSIC}_{\text{linear}}$ in different network architectures. 
Specifically, when the network architecture contain the batch normalization \cite{ioffe2015batch} or $1\times 1$ orthogonal convolution layer \cite{wang2020orthogonal}, we can obtain exactly the same layer-wise importance $\boldsymbol{i}$, and thus obtain a same network architecture. 
Considering that ITPruner has a superior performance in different architectures and datasets, we conclude that $\text{nHSIC}_{\text{linear}}$ is a stable and accurate indicator. 
And ITPruner thus is a robust compression method.

Except for the theorem about the relationship between independence and HSIC in \cite{gretton2005measuring}, we propose a new theorem to illustrate how the mutual information is also minimized by the linear HSIC under a Gaussian assumption.
\begin{theorem}\label{theo:hsic_MI}
	Assuming that $X \sim \mathcal{N} (\boldsymbol{0},  \boldsymbol{\Sigma}_{X}), Y \sim \mathcal{N} (\boldsymbol{0}, \boldsymbol{\Sigma}_{Y})$, $\min \text{nHSIC}_{\text{linear}}(X,Y) \iff \min I(X;Y)$.
\end{theorem}
\begin{proof}
	According to the definition of entropy and mutual information, we have
	\begin{equation}\nonumber
		I(X;Y) = H(X) + H(Y) - H(X, Y).
	\end{equation}
	Meanwhile, according to the definition of the multivariate Gaussian distribution, we have
	\begin{equation}\nonumber
		H(X) = \frac{1}{2}\ln\left( \left(2\pi e \right)^D |\boldsymbol{\Sigma}_X| \right) = \frac{D}{2} (\ln 2 \pi + 1) + \frac{1}{2}\ln|\boldsymbol{\Sigma}_X|,
	\end{equation}
	and $(X,Y)\sim \mathcal{N}(\boldsymbol{0}, \boldsymbol{\Sigma}_{(X,Y)})$, where 
	\begin{equation}\nonumber
	\boldsymbol{\Sigma}_{(X,Y)}=
		\begin{pmatrix}
			\boldsymbol{\Sigma}_{X} & \boldsymbol{\Sigma}_{XY}\\
			\boldsymbol{\Sigma}_{YX} & \boldsymbol{\Sigma}_{Y}\\
		\end{pmatrix}.
	\end{equation}
	Therefore, when $X, Y$ follows a Gaussian distribution, the mutual information is represented as 
	\begin{equation}\label{eq:mutual_info}
		I(X;Y) = \ln|\boldsymbol{\Sigma}_X| + \ln|\boldsymbol{\Sigma}_Y| - \ln|\boldsymbol{\Sigma}_{(X,Y)}|.
	\end{equation}
	Everitt \emph{et al.} \cite{everitt_1958} proposed \textcolor{black}{an} inequation that $|\boldsymbol{\Sigma}_{(X,Y)}|\leq |\boldsymbol{\Sigma}_{X}||\boldsymbol{\Sigma}_{Y}|$, and the equality holds if and only if $ \boldsymbol{\Sigma}_{YX} = \boldsymbol{\Sigma}_{XY}^T = \boldsymbol{X}^T\boldsymbol{Y}$ is a zero matrix. 
	Applying the inequation to Eq.~\ref{eq:mutual_info} we have $I(X;Y)\geq 0$, and the equality holds if and only if $\boldsymbol{X}^T\boldsymbol{Y}$ is a zero matrix. According to the \textcolor{black}{definition} of Frobenius norm, we have $||\boldsymbol{Y}^T \boldsymbol{X} ||_F^2 = ||\boldsymbol{X}^T\boldsymbol{Y} ||_F^2$.
	Apparently, when minimizing Eq.~\ref{eq:linear_HSIC}, we are also minimizing the distance between $\boldsymbol{X}^T\boldsymbol{Y}$ and zero matrix. 
	In other words, while minimizing $\text{nHSIC}_{\text{linear}}$, we are also minimizing the mutual information between two Gaussian distributions, namely, $\min \text{nHSIC}_{\text{linear}}(X,Y) \iff \min I(X;Y).$
	\end{proof}

 Except for the theorem about the relationship between independence and HSIC in \cite{gretton2005measuring}, we propose a new theorem to illustrate how the mutual information is also minimized by the linear HSIC under a Gaussian assumption.
\begin{theorem}\label{theo:hsic_MI}
	Assuming that $X \sim \mathcal{N} (\boldsymbol{0},  \boldsymbol{\Sigma}_{X}), Y \sim \mathcal{N} (\boldsymbol{0}, \boldsymbol{\Sigma}_{Y})$, $\min \text{nHSIC}_{\text{linear}}(X,Y) \iff \min I(X;Y)$.
\end{theorem}
\begin{proof}
	According to the definition of entropy and mutual information, we have
	\begin{equation}\nonumber
		I(X;Y) = H(X) + H(Y) - H(X, Y).
	\end{equation}
	Meanwhile, according to the definition of the multivariate Gaussian distribution, we have
	\begin{equation}\nonumber
		H(X) = \frac{1}{2}\ln\left( \left(2\pi e \right)^D |\boldsymbol{\Sigma}_X| \right) = \frac{D}{2} (\ln 2 \pi + 1) + \frac{1}{2}\ln|\boldsymbol{\Sigma}_X|,
	\end{equation}
	and $(X,Y)\sim \mathcal{N}(\boldsymbol{0}, \boldsymbol{\Sigma}_{(X,Y)})$, where 
	\begin{equation}\nonumber
	\boldsymbol{\Sigma}_{(X,Y)}=
		\begin{pmatrix}
			\boldsymbol{\Sigma}_{X} & \boldsymbol{\Sigma}_{XY}\\
			\boldsymbol{\Sigma}_{YX} & \boldsymbol{\Sigma}_{Y}\\
		\end{pmatrix}.
	\end{equation}
	Therefore, when $X, Y$ follows a Gaussian distribution, the mutual information is represented as 
	\begin{equation}\label{eq:mutual_info}
		I(X;Y) = \ln|\boldsymbol{\Sigma}_X| + \ln|\boldsymbol{\Sigma}_Y| - \ln|\boldsymbol{\Sigma}_{(X,Y)}|.
	\end{equation}
	Everitt \emph{et al.} \cite{everitt_1958} proposed \textcolor{black}{an} inequation that $|\boldsymbol{\Sigma}_{(X,Y)}|\leq |\boldsymbol{\Sigma}_{X}||\boldsymbol{\Sigma}_{Y}|$, and the equality holds if and only if $ \boldsymbol{\Sigma}_{YX} = \boldsymbol{\Sigma}_{XY}^T = \boldsymbol{X}^T\boldsymbol{Y}$ is a zero matrix. 
	Applying the inequation to Eq.~\ref{eq:mutual_info} we have $I(X;Y)\geq 0$, and the equality holds if and only if $\boldsymbol{X}^T\boldsymbol{Y}$ is a zero matrix. According to the \textcolor{black}{definition} of Frobenius norm, we have $||\boldsymbol{Y}^T \boldsymbol{X} ||_F^2 = ||\boldsymbol{X}^T\boldsymbol{Y} ||_F^2$.
	Apparently, when minimizing Eq.~\ref{eq:linear_HSIC}, we are also minimizing the distance between $\boldsymbol{X}^T\boldsymbol{Y}$ and zero matrix. 
	In other words, while minimizing $\text{nHSIC}_{\text{linear}}$, we are also minimizing the mutual information between two Gaussian distributions, namely, $\min \text{nHSIC}_{\text{linear}}(X,Y) \iff \min I(X;Y).$
	\end{proof}

 \begin{figure}[tb]
\centering
\includegraphics[width=1.0\linewidth]{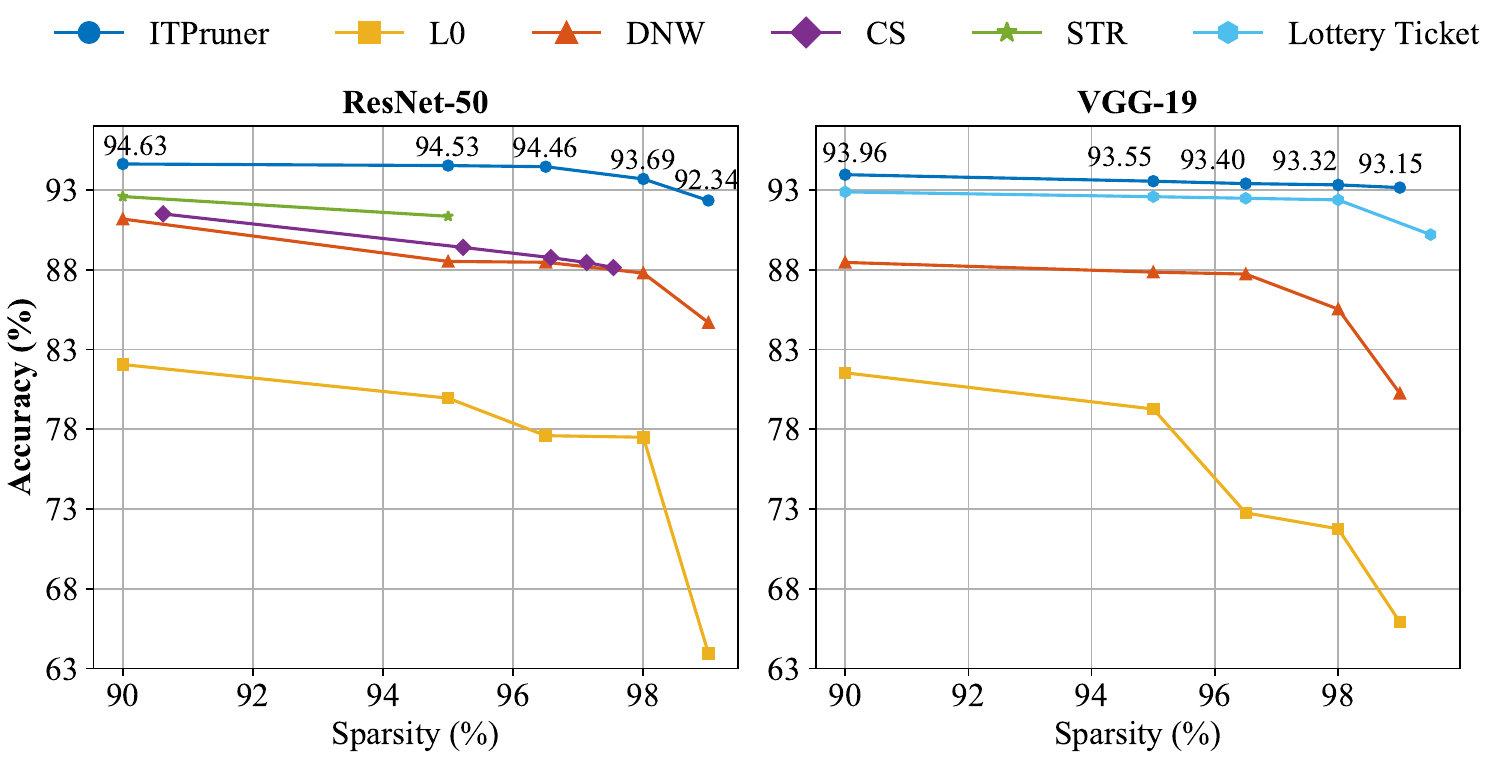}
\caption{{Sparsity ratio vs accuracy tradeoffs for different unstructured pruning methods and ITPruner on ResNet-50 (left) and VGG-19 (right). ITPruner clearly \textcolor{black}{outperforms} the other baselines with a clear margin, especially under large compression ratios.}}
\label{fig:Unstructured_pruning_cifar}
\end{figure}

\begin{table*}[htb]
\begin{center}
\caption{\label{tab:cifar_result} Top $1$ accuracy, compression ratio and search cost of different backbones on CIFAR-10. Specifically, `$\uparrow\downarrow$' denotes the increase and the decrease of accuracy comparing to baseline models, and `Ratio$\downarrow$' indicates the reduction of FLOPs. Notability, metric based methods usually integrate the search with the training process. It is thus hard to recognize search epochs in these methods and `-' stands for unavailable records. The proposed method is emphasized in bold format.}
\resizebox{\textwidth}{!}{
\begin{tabular}{cccccccc}
\toprule
\multirow{2}{*}{\textbf{Model}} & \multirow{2}{*}{\textbf{Method}} & \multirow{2}{*}{\textbf{Type}} & \textbf{Top-1 Acc} & \multirow{2}{*}{$\uparrow\downarrow$} & \textbf{FLOPs} & \multirow{2}{*}{\textbf{Ratio$\downarrow$}} & \textbf{Search}\\
& & & ($\%$) & &\textbf{(M)} & &\textbf{Epoch} \\
\midrule
\multirow{12}{*}{VGG \cite{simonyan2014very}} & Baseline & $-$ & $94.47$ & $-$ & $313.5$ & $-$ & $-$\\
 & L1 \cite{li2016pruning} & Local Metric & $\textcolor{black}{93.40}$ & $1.07\downarrow$ & $\textcolor{black}{206.0}$ & $34.3\%$ & $-$\\
 & FPGM \cite{he2019filter} & Local Metric & $93.54$ & $0.93\downarrow$ & $201.1$ & $35.9\%$ & $-$\\
 & GAL \cite{lin2019towards} & Local Metric & $92.03$ & $2.44\downarrow$ & $189.5$ & $39.6\%$ & $-$\\
 & HRank \cite{lin2020hrank} & Local Metric & $92.34$ & $2.13\downarrow$ & $108.6$ & $65.3\%$ & $-$\\
 & \textcolor{black}{DECODER-200 \cite{DECORE}} & \textcolor{black}{Automated} & $\textcolor{black}{93.56}$ & $\textcolor{black}{0.91\downarrow}$ & $\textcolor{black}{110.3}$ & $\textcolor{black}{64.8\%}$ & $\textcolor{black}{260}$\\
 & \textcolor{black}{FTWT\underline{\hspace{0.5em}}J \cite{FTWT}} & \textcolor{black}{Local Metric} & $\textcolor{black}{93.55}$ & $\textcolor{black}{0.92\downarrow}$ & $\textcolor{black}{109.7}$ & $\textcolor{black}{65.0\%}$ & $\textcolor{black}{-}$\\
 & \textbf{ITPruner} & \textbf{Automated} & $\textbf{94.00}$ & $\textbf{0.47}\downarrow$ & $\textbf{98.8}$ & $\textbf{68.5\%}$ & $\textbf{0}$\\
 \cdashline{2-8}[5.0pt/2pt]
   \specialrule{0em}{1pt}{1pt}
 & Growing \cite{yuan2020growing} & Automated & $\textcolor{black}{92.50}$ & $1.97\downarrow$ & $42.3$ & $86.5\%$ & $60$\\
 & \textcolor{black}{DECODER-100 \cite{DECORE}} & \textcolor{black}{Automated} & $\textcolor{black}{92.44}$ & $\textcolor{black}{2.03\downarrow}$ & $\textcolor{black}{58.0}$ & $\textcolor{black}{81.5\%}$ & $\textcolor{black}{260}$\\
 & \textcolor{black}{FTWT\underline{\hspace{0.5em}}J \cite{FTWT}} & \textcolor{black}{Local Metric} & $\textcolor{black}{92.65}$ & $\textcolor{black}{1.82\downarrow}$ & $\textcolor{black}{81.5}$ & $\textcolor{black}{74.0\%}$ & $\textcolor{black}{-}$\\
  & \textbf{ITPruner} & \textbf{Automated} & $\textbf{93.49}$ & $\textbf{0.98}\downarrow$ & $\textbf{\textcolor{black}{42.4}}$ & $\textbf{86.5}\%$ & $\textbf{0}$\\
 \midrule

\multirow{7}{*}{ResNet-20 \cite{he2016deep}} & Baseline & $-$ & $92.57$ & $-$ & $40.6$ & $-$ \\
 & FPGM \cite{he2019filter} & Local Metric & $91.09$ & $1.48\downarrow$ & $24.3$ & $40.1\%$& $-$ \\
 & APS \cite{wang2020revisiting} & Automated & $91.86$ & $0.71\downarrow$ & $20.9$ & $48.5\%$& 600 \\
 & TAS \cite{dong2019network} & Automated & $91.99$ & $0.58\downarrow$ & $19.9$ & $51.0\%$& 600 \\
  & Taylor \cite{molchanov2019importance} & Global Metric & $91.51$ & $1.06\downarrow$ & $\textcolor{black}{19.3}$ & $52.6\%$& $-$ \\
   & Growing \cite{yuan2020growing} & Automated & $90.91$ & $1.66\downarrow$ & $\textcolor{black}{20.4}$ & $49.8\%$ & $60$\\
 & \textbf{ITPruner} & \textbf{Automated} & $\textbf{92.01}$ & $\textbf{0.56}\downarrow$ & $\textbf{20.8}$ & $\textbf{48.8\%}$ & $\textbf{0}$\\ 
 \midrule
 \multirow{13}{*}{ResNet-56 \cite{he2016deep}} & Baseline & $-$ & $93.93$ & $-$ & $\textcolor{black}{125.0}$ & $-$ \\
   & GAL \cite{lin2019towards} & Local Metric & $92.98$ & $0.95\downarrow$ & $78.3$ & $37.4\%$ & $-$\\
   & APS \cite{wang2020revisiting} & Automated & $93.42$ & $0.51\downarrow$ & $60.3$ & $51.8\%$& 600 \\
 & TAS \cite{dong2019network} & Automated & $92.87$ & $1.06\downarrow$ & $63.1$ & $49.5\%$& 600 \\ 
  & AMC \cite{he2018amc} & Automated & $91.90$ & $2.03\downarrow$ & $62.9$ & $49.7\%$& 400 \\ 
   & \textbf{ITPruner} & \textbf{Automated} & $\textbf{93.43}$ & $\textbf{0.50}\downarrow$ & $\textbf{59.5}$ & $\textbf{52.4\%}$ & $\textbf{0}$\\
   \cdashline{2-8}[5.0pt/2pt]
   \specialrule{0em}{1pt}{1pt}
   & \textcolor{black}{FTWT\underline{\hspace{0.5em}}D \cite{FTWT}} & \textcolor{black}{Local Metric} & $\textcolor{black}{92.63}$ & $\textcolor{black}{1.30\downarrow}$ & $\textcolor{black}{42.5}$ & $\textcolor{black}{66.0\%}$ & $\textcolor{black}{-}$\\
   & \textcolor{black}{FTWT\underline{\hspace{0.5em}}J \cite{FTWT}} & \textcolor{black}{Local Metric} & $\textcolor{black}{92.28}$ & $\textcolor{black}{1.65\downarrow}$ & $\textcolor{black}{57.5}$ & $\textcolor{black}{54.0\%}$ & $\textcolor{black}{-}$\\
    & \textcolor{black}{Random Pruning \cite{Li_random}} & \textcolor{black}{Local Metric} & $\textcolor{black}{93.48}$ & $\textcolor{black}{0.45\downarrow}$ & $\textcolor{black}{63.8}$ & $\textcolor{black}{49.0\%}$ & $\textcolor{black}{-}$\\
   & \textcolor{black}{MFP \cite{He_fiter}} & \textcolor{black}{Global Metric} & $\textcolor{black}{93.56}$ & $\textcolor{black}{0.37\downarrow}$ & $\textcolor{black}{59.3}$ & $\textcolor{black}{52.6\%}$ & $\textcolor{black}{-}$\\
   & \textcolor{black}{DECODER-55 \cite{DECORE}} & \textcolor{black}{Automated} & $\textcolor{black}{90.85}$ & $\textcolor{black}{3.08\downarrow}$ & $\textcolor{black}{23.1}$ & $\textcolor{black}{81.5\%}$ & $\textcolor{black}{260}$\\
   & \textcolor{black}{DECODER-200 \cite{DECORE}} & \textcolor{black}{Automated} & $\textcolor{black}{93.26}$ & $\textcolor{black}{0.67\downarrow}$ & $\textcolor{black}{62.6}$ & $\textcolor{black}{49.9\%}$ & $\textcolor{black}{260}$\\
   & \textbf{ITPruner + FPGM} & \textbf{Automated} & $\textbf{94.05}$ & $\textbf{0.12}\uparrow$ & $\textbf{59.5}$ & $\textbf{52.4\%}$ & $\textbf{0}$\\
\midrule
\multirow{8}{*}{DenseNet-40 \cite{huang2017densely}} &  Baseline &
$-$ &
$94.81$ &
$-$ &
$\textcolor{black}{282.0}$ &
$-$ &
$-$\\

 & GAL-0.01\cite{lin2019towards} & Local Metric & $94.29$ & $0.52\downarrow$ & $\textcolor{black}{182.9}$ & $35.1\%$ & $-$\\
 & HRank \cite{lin2020hrank} & Local Metric & $94.24$ & $0.57\downarrow$ & $\textcolor{black}{167.4}$ & $40.6\%$ & $-$\\
 & \textbf{ITPruner} & \textbf{Automated} & $\textbf{94.56}$ & $\textbf{0.36}\downarrow$ & $\textbf{\textcolor{black}{167.5}}$ & $\textbf{40.6\%}$ & $\textbf{0}$\\
 \cdashline{2-8}[5.0pt/2pt]
   \specialrule{0em}{1pt}{1pt}
   & Zhao \textit{et al} \cite{zhao2019variational} & Local Metric & $93.16$ & $1.65\downarrow$ & $\textcolor{black}{156.0}$ & $44.7\%$ & $-$\\
  & GAL-$0.05$ \cite{lin2019towards} & Local Metric & $93.53$ & $1.28\downarrow$ & $\textcolor{black}{128.1}$ & $54.6\%$ & $-$\\
 &  HRank \cite{lin2020hrank} &  Local Metric &  $93.68$ &  $1.13\downarrow$ &  $\textcolor{black}{110.2}$ &  $60.9\%$ &  $-$\\
  &  \textbf{ITPruner} &  \textbf{Automated} &  $\textbf{94.22}$ &  $\textbf{0.59}\downarrow$ &  $\textbf{\textcolor{black}{110.6}}$ &  $\textbf{60.8\%}$ &  $\textbf{0}$\\
 \bottomrule
 
\end{tabular}}

\end{center}
\end{table*}

\begin{table*}[htb]
\begin{center}
\caption{\label{tab:imagenet_result} Top-$1$ accuracy, compression ratio and search cost of different backbones and methods on ImageNet. Similarly, `$\uparrow\downarrow$' denotes the increase and the decrease of accuracy compared to baseline models, and `Ratio$\downarrow$' indicates the reduction of FLOPs and `-' stands for unavailable records. $^{*}$ denotes the results are reported with knowledge distillation/expanded search space/extra training tricks. The proposed method is emphasized in bold format for better visualization.}
\resizebox{\textwidth}{!}{
\begin{tabular}{cccccccc}
\toprule
\multirow{2}{*}{\textbf{Model}} & \multirow{2}{*}{\textbf{Method}} & \multirow{2}{*}{\textbf{Type}} & \textbf{Top-1 Acc} & \multirow{2}{*}{$\uparrow\downarrow$} & \textbf{FLOPs} & \multirow{2}{*}{\textbf{Ratio$\downarrow$}} & \textbf{Search}\\
& & & ($\%$) & &\textbf{(M)} & &\textbf{Epoch} \\
 \midrule
\multirow{10}{*}{MobileNetV1 \cite{howard2017mobilenets}} & Baseline & $-$ & $\textcolor{black}{72.70}$ & $-$ & $569$ & $-$ & $-$\\
 & Uniform ($0.75\times$) & $-$ & $\textcolor{black}{68.40}$ & $\textcolor{black}{4.30}\downarrow$ & $325$ & $42.9\%$ & $-$\\
 & NetAdapt \cite{yang2018netadapt} & Automated & $\textcolor{black}{69.10}$ & $\textcolor{black}{3.60}\downarrow$ & $284$ & $50.1\%$ & $-$\\
  & AMC \cite{he2018amc} & Automated & $70.50$ & $2.20\downarrow$ & $285$ & $49.9\%$& $400$ \\
  & \textcolor{black}{FTWT \cite{FTWT}} & \textcolor{black}{Local Metric} & $\textcolor{black}{69.66}$ & $\textcolor{black}{3.04\downarrow}$ & $\textcolor{black}{335}$ & $\textcolor{black}{41.1\%}$ & $\textcolor{black}{-}$\\
  & \textbf{ITPruner} & \textbf{Automated} & $\textbf{70.92}$ & $\textbf{1.78}\downarrow$ & $\textbf{283}$ & $\textbf{50.3\%}$ & $\textbf{0}$\\ 
  \cdashline{2-8}[5.0pt/2pt]
   \specialrule{0em}{1pt}{1pt}
   & Uniform ($0.5\times$) & $-$ & $\textcolor{black}{63.70}$ & $\textcolor{black}{9.00}\downarrow$ & $149$ & $73.8\%$ & $-$\\
 & MetaPruning$^*$ \cite{liu2019metapruning} & Automated & $\textcolor{black}{66.10}$ & $\textcolor{black}{6.60}\downarrow$ & $149$ & $73.8\%$ & $60$\\
 & AutoSlim$^*$ \cite{yu2019autoslim} & Automated & $67.90$ & $4.80\downarrow$ & $150$ & $73.6\%$ & $100$\\
 & \textbf{ITPruner} & \textbf{Automated} & $\textbf{68.06}$ & $\textbf{4.64}\downarrow$ & $\textbf{149}$ & $\textbf{73.8\%}$ & $\textbf{0}$\\ 
 \midrule
\multirow{9}{*}{MobileNetV2 \cite{sandler2018mobilenetv2}} & Baseline & $-$ & $\textcolor{black}{72.10}$ & $-$ & $301$ & $-$ \\
 & Uniform ($0.75\times$)& $-$ & $\textcolor{black}{69.80}$ & $\textcolor{black}{2.30}\downarrow$ & $220$ & $26.9\%$& $-$ \\
 & AMC \cite{he2018amc} & Automated & $\textcolor{black}{70.80}$ & $\textcolor{black}{1.30}\downarrow$ & $220$ & $26.6\%$& $400$ \\
 & MetaPruning$^*$ \cite{liu2019metapruning} & Automated & $71.20$ & $\textcolor{black}{0.90}\downarrow$ & $217$ & $27.9\%$ & $60$\\
 & AutoSlim$^*$ \cite{yu2019autoslim} & Automated & $71.14$ & $0.96\downarrow$ & $207$ & $31.2\%$ & $100$\\
 &\color{black}  Random Pruning \cite{Li_random} &\color{black}  Local Metric &\color{black}  $\textcolor{black}{70.90}$ &\color{black}  $\textcolor{black}{1.20}\downarrow$ &\color{black}  $213$ &\color{black}  $29.1\%$ &\color{black}  120\\
  & \textbf{ITPruner} & \textbf{Automated} & $\textbf{71.54}$ & $\textbf{0.56}\downarrow$ & $\textbf{219}$ & $\textbf{27.2\%}$ & $\textbf{0}$\\
  \cdashline{2-8}[5.0pt/2pt]
   \specialrule{0em}{1pt}{1pt}
 & APS \cite{wang2020revisiting} & Automated & $68.96$ & $3.14\downarrow$ & $156$ & $48.2\%$& $160$ \\
 & \textbf{ITPruner} & \textbf{Automated} & $\textbf{69.13}$ & $\textbf{2.97}\downarrow$ & $\textbf{149}$ & $\textbf{50.5\%}$ & $\textbf{0}$\\ 
 \midrule
 \multirow{18}{*}{ResNet-50 \cite{he2016deep}} & Baseline & $-$ & $76.88$ & $-$ & $4089$ & $-$ \\
 & ABCPruner \cite{linchannel} & Automated & $74.84$ & $2.04\downarrow$ & $2560$ & $37.4\%$ & $12$\\
    & GAL \cite{lin2019towards} & Local Metric & $71.95$ & $4.93\downarrow$ & $2330$ & $43.0\%$ & $-$\\
  & Taylor \cite{molchanov2019importance} & Global Metric & $\textcolor{black}{74.50}$ & $2.38\downarrow$ & $2250$ & $45.0\%$ & $-$\\
  & HRank \cite{lin2020hrank} & Local Metric & $74.98$ & $1.90\downarrow$ & $2300$ & $43.8\%$ & $-$\\
  &\color{black}  Random Pruning \cite{Li_random} &\color{black}  Local Metric &\color{black}  $74.15$ &\color{black}  $2.73\downarrow$ &\color{black}  $2074$ &\color{black}  $\textcolor{black}{49.3}\%$ &\color{black}  25\\
   &\color{black}  Random Pruning \cite{Li_random} &\color{black}  Local Metric &\color{black}  $74.78$ &\color{black}  $\textcolor{black}{2.10}\downarrow$ &\color{black}  $2074$ &\color{black}  $\textcolor{black}{49.3\%}$ &\color{black}  75\\
   &\color{black}  Random Pruning \cite{Li_random} &\color{black}  Local Metric &\color{black}  $75.13$ &\color{black}  $1.75\downarrow$ &\color{black}  $2074$ &\color{black}  $\textcolor{black}{49.3\%}$ &\color{black}  120\\
   &\color{black}  MFP \cite{He_fiter} &\color{black}  Global Metric &\color{black}  $74.86$ &\color{black}  $2.02\downarrow$ &\color{black}  $1942$ &\color{black}  $53.5\%$ &\color{black}  $-$\\
   & \textcolor{black}{DECODER-6 \cite{DECORE}} & \textcolor{black}{Automated} & $\textcolor{black}{74.58}$ & $\textcolor{black}{\textcolor{black}{2.30}\downarrow}$ & $\textcolor{black}{2372}$ & $\textcolor{black}{42.0\%}$ & $\textcolor{black}{20}$\\
   & ITPruner & Automated & $75.28$ & $1.60\downarrow$ & $1943$ & $52.5\%$ & $0$\\
   & \textbf{ITPruner} & Automated & $\textbf{75.75}$ & $\textbf{1.13}\downarrow$ & $\textbf{2236}$ & $\textbf{45.3\%}$ & $\textbf{0}$\\
   \cdashline{2-8}[5.0pt/2pt]
   \specialrule{0em}{1pt}{1pt}
   & AutoSlim$^*$ \cite{yu2019autoslim} & Automated & $\textcolor{black}{74.90}$ & $1.98\downarrow$ & $2300$ & $43.8\%$ & $100$\\
    & TAS$^*$ \cite{dong2019network} & Automated & $76.20$ & $0.68\downarrow$ & $2310$ & $43.5\%$ & $600$\\
    & MetaPruning$^*$ \cite{liu2019metapruning} & Automated & $72.17$ & $4.71\downarrow$ & $2260$ & $44.7\%$ & $60$\\
    & FPGM$^*$ \cite{he2019filter} & Local Metric & $75.59$ & $1.29\downarrow$ & $2167$ & $47.0\%$ & $-$\\
   &\color{black}  DS-MBNet-S++$^*$ \cite{Li_DSNet} &\color{black}  Automated &\color{black}  $\textcolor{black}{76.40}$ &\color{black}  $0.48\downarrow$ &\color{black}  $2300$ &\color{black}  $43.8\%$ &\color{black}  $160$\\
   & \textbf{ITPruner}$^*$ & Automated & $\textbf{78.05}$ & $\textbf{1.17}\uparrow$ & $\textbf{1943}$ & $\textbf{52.5\%}$ & $\textbf{0}$\\
\midrule
\multirow{4}{*}{\textcolor{black}{DeiT \cite{touvron2021training}}} & \textcolor{black}{Baseline} & $\textcolor{black}{-}$ & $\textcolor{black}{81.80}$ & $\textcolor{black}{-}$ & $\textcolor{black}{16.8}$ & $\textcolor{black}{-}$ & $\textcolor{black}{-}$\\
 & \textcolor{black}{IA-RED$^2$ \cite{pan2021ia}} & \textcolor{black}{Automated} & $\textcolor{black}{80.30}$ & $\textcolor{black}{1.50\downarrow}$ & $\textcolor{black}{11.8}$ & $\textcolor{black}{29.8\%}$ & $\textcolor{black}{90}$\\
 & \textcolor{black}{WDPruning \cite{yu2022width}} & \textcolor{black}{Automated} & $\textcolor{black}{80.76}$ & $\textcolor{black}{1.04\downarrow}$ & $\textcolor{black}{9.9}$ & $\textcolor{black}{41.1\%}$ & $\textcolor{black}{100}$\\
  & \textcolor{black}{\textbf{ITPruner}} & \textcolor{black}{\textbf{Automated}} & $\textcolor{black}{\textbf{81.20}}$ & $\textcolor{black}{\textbf{0.60}\downarrow}$ & $\textcolor{black}{\textbf{9.9}}$ & $\textcolor{black}{\textbf{41.1}\%}$ & $\textcolor{black}{\textbf{0}}$\\
 \bottomrule
\end{tabular}
}
\end{center}

\end{table*}

\begin{figure*} [htb]
\centering
\includegraphics[width=1.0\linewidth]{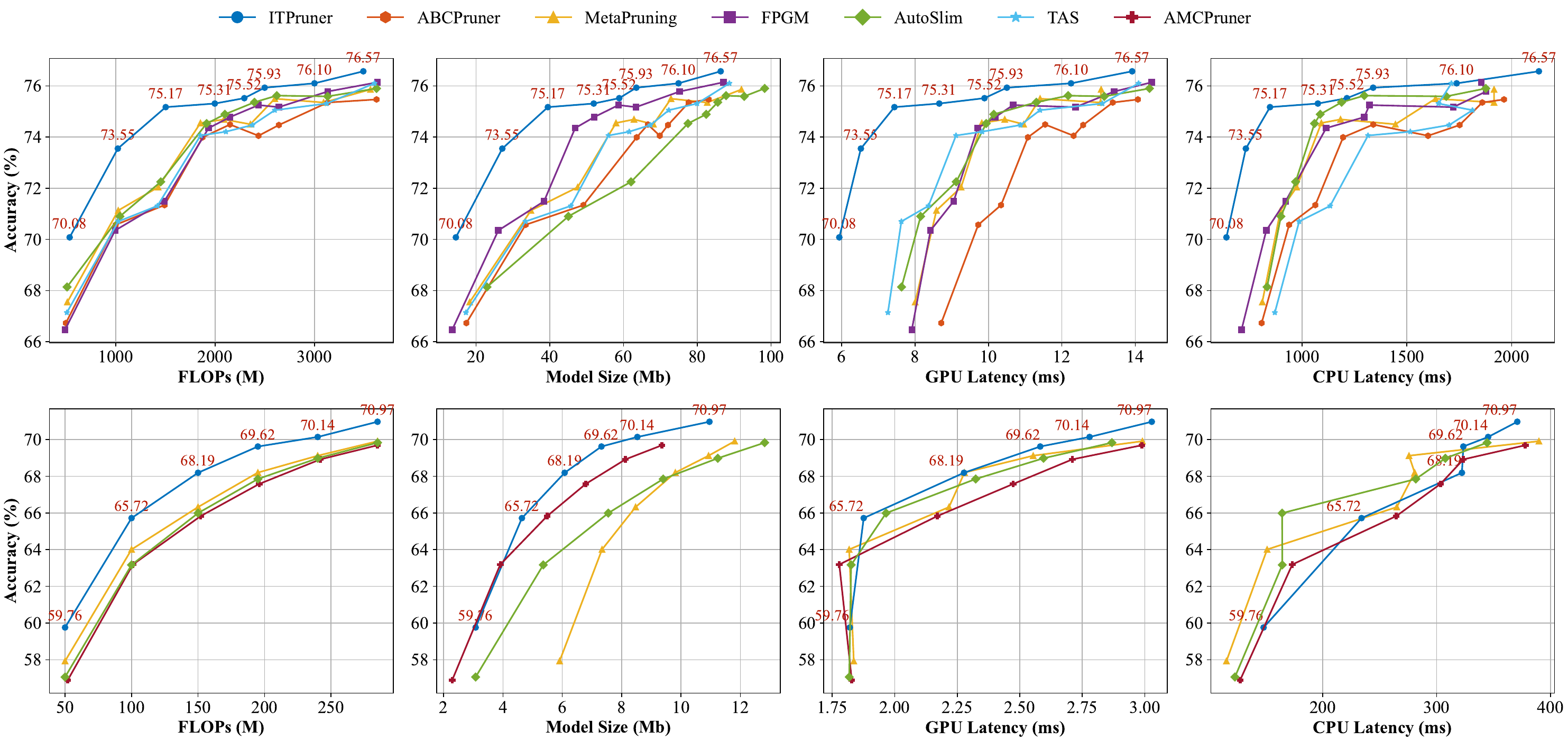}
\caption{{FLOPs, Size, GPU Latency and CPU Latency vs accuracy tradeoffs for different pruning methods and ITPruner on ResNet-50 (up) and MobileNetV1 (bottom). All the models are searched or adjusted according to FLOPs. ITPruner clearly \textcolor{black}{outperforms} the other baselines with a clear margin in most cases.}}
\label{fig:res50_mbv1_pruning}
\end{figure*}

\begin{table}[]
\caption{Inference time under different batch sizes of ITPruner compressed networks and the corresponding baselines on a mobile CPU of Google Pixel $2$. Our method is emphasized in bold format.}
\setlength{\tabcolsep}{1mm}{
\begin{tabular}{ccccc}
\toprule
\multirow{2}{*}{Model} & \multirow{2}{*}{Method} & \multicolumn{3}{c}{Batch Size} \\ 
\cmidrule{3-5}
 &  & 1 & 4 & 8 \\ \midrule
\multirow{2}{*}{MobileNetV1} & Baseline & $28$ms & $125$ms & $252$ms \\
 & \textbf{Ours($\textbf{283}$M)} & $\textbf{15}$\textbf{ms} & $\textbf{70}$\textbf{ms} & $\textbf{135}$\textbf{ms} \\ 
 & \textbf{Ours($\textbf{149}$M)} & $\textbf{10}$\textbf{ms} & $\textbf{45}$\textbf{ms} & $\textbf{84}$\textbf{ms} \\ \midrule
\multirow{2}{*}{MobileNetV2} & Baseline & $21$ms & $94$ms & $182$ms \\
 & \textbf{Ours($\textbf{217}$M)} & $\textbf{17}$\textbf{ms} & $\textbf{75}$\textbf{ms} & $\textbf{144}$\textbf{ms} \\ 
 & \textbf{Ours($\textbf{148}$M)} & $\textbf{13}$\textbf{ms} & $\textbf{55}$\textbf{ms} & $\textbf{110}$\textbf{ms} \\ 
 \bottomrule
\end{tabular}}
\label{tab:time_cost}
\end{table}

\section{Experiment}\label{sec:expri}

We quantitatively demonstrate the robustness and efficiency of ITPruner in this section. 
We first describe the detailed experimental settings of the proposed algorithm in Sec.~\ref{subsec:experiset}. 
Then, we apply ITPruner to the {automated} network pruning for image classification on the widely-used CIFAR-10 \cite{lecun1998gradient} and ImageNet \cite{russakovsky2015imagenet} datasets with different constraints in Sec.~\ref{subsec:comp_sota}. 
To further understand the ITPruner, we conduct a few ablation studies to show the effectiveness of our algorithm in Sec.~\ref{subsec:abla}. 
The experiments are done with NVIDIA Tesla V100, and the algorithm is implemented using PyTorch \cite{paszke2017automatic}. 
We have also released all the source code. 

\subsection{Experimental Settings}\label{subsec:experiset}

\textcolor{black}{All of our experiments are conducted on PyTorch with Nvidia 3090 and A100 GPUs.} We use the same datasets and evaluation {metrics} with existing compression methods \cite{yang2018netadapt, wang2020revisiting, dong2019network, liu2019metapruning, yu2019autoslim, he2018amc}. 
First, most of the experiments are conducted on \textcolor{black}{image classification benchmark} CIFAR-10, which has $50$K training images and $10$K testing images from 10 classes with a resolution of $32 \times 32$. 
The color intensities of all images are normalized to $[-1, +1]$. 
To further evaluate the generalization capability, we also evaluate the classification accuracy on \textcolor{black}{another image classification benchmark} ImageNet, which consists of $1,000$ classes with $1.28$M training images and $50$K validation images. 
Here, we consider the input image size is $224 \times 224$. 
We compare different methods under similar baselines, training conditions and search spaces in our experiment. 
We elaborate on training conditions as follows. 

For the standard training, all networks are trained via the SGD with a momentum $0.9$. We train the founded network over $300$ epochs in CIFAR-10. In ImageNet, we train $120$ and $250$ epochs for ResNet and MobileNet, respectively \footnote{Such setting is consistent with the previous compression works for a fair comparison.}. We set the batch size of $256$ and the initial learning rate of $0.1$ and cosine learning rate strategy in both datasets. 
We also use basic data augmentation strategies in both datasets: images are randomly cropped and flipped, then resized to the corresponding input sizes as we mentioned before and finally normalized with their mean and standard deviation.

\subsection{Comparison with State-of-the-Art Methods}\label{subsec:comp_sota}

\textbf{CIFAR-10}. We first compare the proposed ITPruner with different types of compression methods on VGG, ResNet-20 and ResNet-56. 
Results are summarized in Tab.~\ref{tab:cifar_result}. 
Obviously, \emph{without any search cost}, our method achieves the best compression trade-off compared to other methods. 
Specifically, ITPruner shows a larger reduction of FLOPs but with better performance. 
For example, compared to the local metric method HRank \cite{lin2020hrank}, ITPruner achieves higher reductions of FLOPs ($68.5\%$ vs $65.3\%$) with higher top-1 accuracy ($94.00$ vs $92.34$). 
Meanwhile, compared to search based method TAS \cite{dong2019network}, ITPruner yields better performance on ResNet-56 ($52.4\%$ vs. $49.5\%$ in FLOPs reduction, and $93.43$ vs $92.87$ in top-1 accuracy). 
Besides, the proposed ITPruner is orthonormal to the local metric based method, which means that ITPruner is capable of integrating these methods to achieve better performance. 
As we can see, integrating FPGM \cite{he2019filter} with ITPruner achieves a further improvement in accuracy, which is even better than the baseline. 
Another interesting observation of Tab.~\ref{tab:cifar_result} is the performance rank between search and local metric based methods. 
In particular, search based methods show a significant performance gap on the efficient model like ResNet-20, while a worse performance on large models such as VGG and ResNet-56.
Such an observation provides a suggestion on how to select compression tools for different models. 

{As we mentioned before, ITPruner is orthogonal to traditional pruning methods, including structured pruning and unstructured pruning. We thus further conduct extensive experiments by integrating ITPruner with network sparsification, such as LTH~\cite{frankle2018lottery}, CS~\cite{savarese2020winning}, L$0$~\cite{louizos2018learning}, STR~\cite{kusupati2020soft} and DNW~\cite{wortsman2019discovering}. As illustrated in Fig.~\ref{fig:Unstructured_pruning_cifar}, our method still achieves the best accuracy under different sparsity ratios. For example, ITPruner yields $99\%$ compression rate with only a decrease of $2.02\%$ in Top-1 accuracy on CIFAR-10, which is $7.65\%$ and $28.38\%$ lower than DNW and L$0$. }

\textbf{ImageNet 2012}. We further compare our ITPruner scheme with other methods on widely-used ResNet \cite{he2016deep}, MobileNetV1\cite{howard2017mobilenets} and MobileNetV2 \cite{sandler2018mobilenetv2}. 
As shown in Tab.~\ref{tab:imagenet_result}, our method still achieves the best trade-off compared to different types of methods. 
Specifically, ITPruner yields $2.21\times$ compression rate with only a decrease of $1.13\%$ in Top-1 accuracy on ResNet-50. 
Similar results are reported by Tab.~\ref{tab:imagenet_result} for other backbones.

In particular, the proposed method shows a clear performance gap compared to other methods when compressing compact models MobileNetV1 and MobileNetV2.
For example, ITPruner achieves a similar reduction of FLOPs with a much lower accuracy drop ($0.8$ vs $1.3$) in MobileNetV1 compared to widely used AMC \cite{he2018amc}. 
Moreover, we also demonstrate the effectiveness of networks adapted by ITPruner on Google Pixel $2$. As shown in Tab.~\ref{tab:time_cost}, ITPruner achieves $3\times$ and $1.65\times$ acceleration rates with batch size $8$ on MobileNetV1 and MobileNetV2, respectively. 

{Notably, we use $^*$ to denote the results reported with knowledge distillation/expanded search space/extra training tricks in Tab.~\ref{tab:imagenet_result}. 
Meanwhile, ITPruner without any tricks shows comparable performance to these methods, which implicitly demonstrates the effectiveness of our method. 
Moreover, incorporating with these tricks, ITPruner shows a significant improvement ($78.09$ vs $75.59$) with a lower FLOPs ($1,943$ vs $2,167$). 
Therefore, to make a fair comparison, we first collect the codes of the compared baselines. 
And then, all the methods are searched in the same search space. 
The compressed models that show similar performance to the original official codes are selected. 
After that, the searched models are trained under the same training conditions. Specifically, following most previous works, we train $120$ and $250$ epochs for ResNet-50 and MobileNetV1, respectively. 
Meanwhile, the other settings are directly adopted from widely-used Torchvision. Note that all the models are searched with FLOPs constraints. 
In other words, it is fairer to compare different compressed models with such metric. 
Meanwhile, the GPU and CPU Latency is extremely unstable in MobileNetV1, which are highly influenced by the hyperparameters such as batchsize. 
As illustrated in Fig.~\ref{fig:res50_mbv1_pruning}, ITPruner clearly outperforms the other baselines by a large margin in most cases.} \textcolor{black}{To further expand the experimental scope, we adopt ITPruner on ViT compression, where the results are demonstrated in Tab.~\ref{tab:imagenet_result}. As we can see, ITPruner is still effective in ViT pruning. That is, ITPruner is capable of compressing $41\%$ FLOPs with only a $0.6$ performance drop on DeiT \cite{touvron2021training}, which is $0.9$ and $0.44$ better than the previous baselines \cite{pan2021ia, yu2022width}.}

\begin{table*}[h]
\color{black}
\caption{Top $1$ accuracy, compression ratio and search cost of DeiT and Swin on ImageNet. Specifically, `$\uparrow\downarrow$' denotes the increase and the decrease of accuracy comparing to baseline models, and `Ratio$\downarrow$' indicates the reduction of FLOPs. Notability, metric based methods usually integrate the search with the training process. It is thus hard to recognize search epochs in these methods and `-' stands for unavailable records. The proposed method is emphasized in bold format.}
\begin{center}
\setlength{\tabcolsep}{0.5mm}{
\begin{tabular}{cccccccc}
\toprule
\multirow{2}{*}{\textbf{Model}} & \multirow{2}{*}{\textbf{Method}} & \multirow{2}{*}{\textbf{Type}} & \textbf{Top-1 Acc} & \multirow{2}{*}{$\uparrow\downarrow$} & \textbf{FLOPs} & \multirow{2}{*}{\textbf{Ratio$\downarrow$}} & \textbf{Search}\\
& & & ($\%$) & &\textbf{(G)} & &\textbf{Epoch} \\
\midrule
\multirow{5}{*}{\textcolor{black}{DeiT-B \cite{touvron2021training}}} & \textcolor{black}{Baseline} & $\textcolor{black}{-}$ & $\textcolor{black}{81.8}$ & $\textcolor{black}{-}$ & $\textcolor{black}{16.8}$ & $\textcolor{black}{-}$ & $\textcolor{black}{-}$\\
& \textcolor{black}{IA-RED$^2$ \cite{pan2021ia}} & \textcolor{black}{Automated} & $\textcolor{black}{80.3}$ & $\textcolor{black}{1.5\downarrow}$ & $\textcolor{black}{11.8}$ & $\textcolor{black}{29.8\%}$ & $\textcolor{black}{90}$\\
& \textcolor{black}{WDPruning \cite{yu2022width}} & \textcolor{black}{Automated} & $\textcolor{black}{80.76}$ & $\textcolor{black}{1.04\downarrow}$ & $\textcolor{black}{9.9}$ & $\textcolor{black}{41.1\%}$ & $\textcolor{black}{100}$\\
& \textcolor{black}{NViT \cite{Yang2023nvit}} & \textcolor{black}{Automated} & $\textcolor{black}{83.29}$ & $\textcolor{black}{1.49\uparrow}$ & $\textcolor{black}{6.8}$ & $\textcolor{black}{59.5\%}$ & $\textcolor{black}{300}$\\
& \textcolor{black}{\textbf{ITPruner}} & \textcolor{black}{\textbf{Automated}} & $\textcolor{black}{\textbf{83.30}}$ & $\textcolor{black}{\textbf{1.50}\uparrow}$ & $\textcolor{black}{\textbf{6.8}}$ & $\textcolor{black}{\textbf{59.5}\%}$ & $\textcolor{black}{\textbf{0}}$\\
\midrule
\multirow{5}{*}{\textcolor{black}{DeiT-S \cite{touvron2021training}}} & \textcolor{black}{Baseline} & $\textcolor{black}{-}$ & $\textcolor{black}{79.8}$ & $\textcolor{black}{-}$ & $\textcolor{black}{4.6}$ & $\textcolor{black}{-}$ & $\textcolor{black}{-}$\\
& \textcolor{black}{IA-RED$^2$ \cite{pan2021ia}} & \textcolor{black}{Automated} & $\textcolor{black}{79.1}$ & $\textcolor{black}{0.7\downarrow}$ & $\textcolor{black}{-}$ & $\textcolor{black}{-}$ & $\textcolor{black}{90}$\\
& \textcolor{black}{WDPruning \cite{yu2022width}} & \textcolor{black}{Automated} & $\textcolor{black}{78.38}$ & $\textcolor{black}{1.42\downarrow}$ & $\textcolor{black}{2.6}$ & $\textcolor{black}{43.5\%}$ & $\textcolor{black}{100}$\\
& \textcolor{black}{NViT \cite{Yang2023nvit}} & \textcolor{black}{Automated} & $\textcolor{black}{82.19}$ & $\textcolor{black}{2.39\uparrow}$ & $\textcolor{black}{4.2}$ & $\textcolor{black}{8.7\%}$ & $\textcolor{black}{300}$\\
& \textcolor{black}{\textbf{ITPruner}} & \textcolor{black}{\textbf{Automated}} & $\textcolor{black}{\textbf{82.50}}$ & $\textcolor{black}{\textbf{2.70}\uparrow}$ & $\textcolor{black}{\textbf{4.2}}$ & $\textcolor{black}{\textbf{8.7}\%}$ & $\textcolor{black}{\textbf{0}}$\\
\midrule
\multirow{5}{*}{\textcolor{black}{DeiT-T \cite{touvron2021training}}} & \textcolor{black}{Baseline} & $\textcolor{black}{-}$ & $\textcolor{black}{74.50}$ & $\textcolor{black}{-}$ & $\textcolor{black}{1.3}$ & $\textcolor{black}{-}$ & $\textcolor{black}{-}$\\
& \textcolor{black}{WDPruning \cite{yu2022width}} & \textcolor{black}{Automated} & $\textcolor{black}{70.34}$ & $\textcolor{black}{4.16\downarrow}$ & $\textcolor{black}{0.7}$ & $\textcolor{black}{46.2\%}$ & $\textcolor{black}{100}$\\
& \textcolor{black}{NViT \cite{Yang2023nvit}} & \textcolor{black}{Automated} & $\textcolor{black}{76.21}$ & $\textcolor{black}{1.71\uparrow}$ & $\textcolor{black}{1.3}$ & $\textcolor{black}{0\%}$ & $\textcolor{black}{300}$\\
& \textcolor{black}{\textbf{ITPruner}} & \textcolor{black}{\textbf{Automated}} & $\textcolor{black}{\textbf{76.22}}$ & $\textcolor{black}{\textbf{1.72}\uparrow}$ & $\textcolor{black}{\textbf{1.3}}$ & $\textcolor{black}{\textbf{0}\%}$ & $\textcolor{black}{\textbf{0}}$\\
\midrule
\multirow{3}{*}{\textcolor{black}{Swin-S \cite{Liu2021swin}}} & \textcolor{black}{Baseline} & $\textcolor{black}{-}$ & $\textcolor{black}{83.00}$ & $\textcolor{black}{-}$ & $\textcolor{black}{8.7}$ & $\textcolor{black}{-}$ & $\textcolor{black}{-}$\\
& \textcolor{black}{NViT \cite{Yang2023nvit}} & \textcolor{black}{Automated} & $\textcolor{black}{82.95}$ & $\textcolor{black}{0.05\downarrow}$ & $\textcolor{black}{6.2}$ & $\textcolor{black}{28.7\%}$ & $\textcolor{black}{300}$\\
& \textcolor{black}{\textbf{ITPruner}} & \textcolor{black}{\textbf{Automated}} & $\textcolor{black}{\textbf{82.90}}$ & $\textcolor{black}{\textbf{0.1}\downarrow}$ & $\textcolor{black}{\textbf{6.2}}$ & $\textcolor{black}{\textbf{28.7}\%}$ & $\textcolor{black}{\textbf{0}}$\\
\bottomrule
\end{tabular}}
\end{center}
\label{tab:imagenet_result_transformer1}
\end{table*}

\textcolor{black}{
\textbf{Latest vision transformer pruning papers.} Recent advances in vision transformer pruning have demonstrated significant progress in model compression while maintaining or even improving model performance. We conducted extensive experiments comparing ITPruner with several state-of-the-art pruning methods, including NViT \cite{Yang2023nvit}, IA-RED$^2$ \cite{pan2021ia}, and WDPruning \cite{yu2022width}. While NViT, the current state-of-the-art approach, employs a sophisticated Hessian-based structural pruning criterion combined with latency-aware regularization to optimize module-wise pruning rates, our proposed ITPruner takes a fundamentally different approach by leveraging the information bottleneck principle.}

\textcolor{black}{
The experimental results, as detailed in Table~\ref{tab:imagenet_result_transformer1}, demonstrate ITPruner's superior performance across various model architectures. For DeiT-B, ITPruner achieves a remarkable 1.50\% accuracy improvement while reducing FLOPs by 59.5\%. Similarly, on DeiT-S, ITPruner surpasses all competing methods with a 2.70\% accuracy gain and 8.7\% FLOPs reduction. Even for the compact DeiT-T architecture, ITPruner maintains its advantage with a 1.72\% accuracy improvement.}

\textcolor{black}{
Notably, when compared to NViT, ITPruner consistently demonstrates better or comparable performance. For instance, ITPruner achieves higher ImageNet-1k accuracy on DeiT-S (+0.31\%) while matching the FLOPs reduction. On Swin-S, although both methods show slight accuracy degradation, ITPruner maintains competitive performance (-0.1\%) while achieving the same 28.7\% FLOPs reduction.}

\textcolor{black}{
A particularly distinctive advantage of ITPruner is its computational efficiency during the pruning process. While other methods require significant search epochs (e.g., NViT requires 300 epochs, WDPruning needs 100 epochs), ITPruner achieves these superior results with zero search epochs, demonstrating its remarkable efficiency in the pruning process. This characteristic makes ITPruner particularly attractive for practical applications where computational resources and time constraints are important considerations.}

\textcolor{black}{
These comprehensive results, as presented in Table~\ref{tab:imagenet_result_transformer1}, establish ITPruner as a highly effective and efficient approach for vision transformer pruning, offering a compelling combination of accuracy improvement, model compression, and computational efficiency.
}

\textcolor{black}{
\textbf{Comprehensive Evaluation Across Vision Transformers.} We conducted extensive experiments to evaluate ITPruner's effectiveness across a diverse range of transformer-based vision models, encompassing various architectures, scales, and pre-training paradigms. Our results demonstrate consistent superior performance while maintaining remarkable computational efficiency.}

\textcolor{black}{
In the context of self-supervised pre-trained Vision Transformers (ViTs), ITPruner demonstrates exceptional performance across different model scales. For tiny-scale models (5.8M parameters), ITPruner achieves an impressive 81.6\% top-1 accuracy on ImageNet-1K, substantially outperforming established methods such as MAE (71.6\%), MoCo (73.3\%), and TinyMIM (75.8\%). This performance advantage extends to semantic segmentation tasks, where ITPruner achieves 45.7\% mIoU on ADE20K, surpassing comparable methods including SparseMAE (45.2\%) and MAE-lite (43.9\%). For small-scale models (11.3M parameters), ITPruner maintains competitive performance with 82.4\% top-1 accuracy, comparing favorably to more complex approaches like TinyMIM (83.0\%) and G2SD (82.0\%), while requiring no search epochs.}

\textcolor{black}{
In the domain of vision-language models, particularly CLIP-based architectures, ITPruner achieves state-of-the-art results across multiple benchmarks. On the MSCOCO dataset, ITPruner sets new standards with 58.4\% TR@1 and 41.4\% IR@1, surpassing previous best results. Similarly, on Flickr30K, ITPruner achieves exceptional performance with 86.7\% TR@1 and 70.0\% IR@1, while maintaining competitive performance with reduced model dimensions. These results are particularly noteworthy when compared to larger models like CLIP-ViT-L/14 and other efficient variants such as EfficientVLM and TinyCLIP.}

\textcolor{black}{
For downstream object detection tasks, ITPruner demonstrates remarkable efficiency while maintaining competitive accuracy. When applied to DINO with a Swin-T backbone, ITPruner achieves 54.0 AP on COCO detection, virtually matching DINO's performance (54.1 AP) while dramatically reducing encoder GFLOPs from 137 to 36. This significant efficiency improvement is achieved without compromising performance across various object scales (AP$_S$=38.0, AP$_M$=57.0, AP$_L$=69.2). Similar efficiency gains are observed with the ResNet-50 backbone, where ITPruner maintains competitive performance (50.3 AP) while reducing computational overhead substantially.}

\textcolor{black}{
A particularly distinctive feature of ITPruner is its ability to achieve these results without requiring any search epochs, in stark contrast to other methods that typically require hundreds of epochs (e.g., MAE: 1600 epochs, TinyMIM: 300 epochs). This characteristic makes ITPruner exceptionally practical for resource-constrained scenarios while maintaining competitive or superior performance across different architecture scales (tiny to large) and various pre-training approaches (CLIP, MAE, DINO).}

\textcolor{black}{
Detailed experimental results presented in Tables \ref{itpruner_clip}, \ref{itpruner_mae}, and \ref{itpruner_dino} consistently demonstrate ITPruner's effectiveness in balancing model performance and computational efficiency across diverse vision tasks. The method's ability to maintain or improve performance while significantly reducing computational requirements, coupled with its zero-epoch search requirement, establishes ITPruner as a practical and powerful solution for efficient vision transformer deployment.
}

\textcolor{black}{
\textbf{Evaluation on Vision-Language Models.} We conducted comprehensive experiments to evaluate ITPruner's effectiveness across various multimodal scenarios and language model configurations, demonstrating its exceptional capability in maintaining model performance while achieving significant parameter reduction.}

\textcolor{black}{
In the context of vision-language models, our experiments with LaVIN-7B on the ScienceQA dataset reveal remarkable results across different pruning ratios. With 30\% pruning, ITPruner achieves an impressive 90.29\% average accuracy, substantially outperforming the baseline LaVIN-7B-MaP (85.85\%). This strong performance is maintained even under more aggressive 50\% pruning, where ITPruner achieves 89.39\% accuracy, nearly matching the dense model's 89.41\% performance while using only half the parameters.}

\textcolor{black}{
The subject-wise analysis demonstrates ITPruner's consistent excellence across diverse domains. With 30\% pruning, it achieves 90.05\% accuracy in natural sciences, 94.83\% in social sciences, and 87.09\% in language sciences, surpassing both specialized models and larger language models. Notably, these results exceed GPT-4's performance in several subject areas while consistently outperforming GPT-3.5, both with and without Chain-of-Thought prompting.}

\textcolor{black}{
ITPruner demonstrates particularly robust cross-modal understanding capabilities. At 30\% pruning, it achieves 89.24\% accuracy on image contexts and 89.98\% on text contexts, surpassing specialized vision-language models including LLaVA (88.00\% on image contexts) and LaVIN-7B-MMA (87.46\% on image contexts). This performance is especially noteworthy given the reduced parameter count and demonstrates ITPruner's ability to maintain strong cross-modal understanding even under significant pruning.}

\textcolor{black}{
In the context of large language model pruning, as shown in Table~\ref{tab:itpruner_on_LLM}, ITPruner demonstrates superior performance across both structured and unstructured pruning scenarios. When combined with existing pruning methods, ITPruner consistently improves their performance. For instance, integrating ITPruner with Wanda-sp at 30\% structured pruning improves performance across all evaluation metrics, achieving better perplexity on WikiText2 (16.25 vs 24.53) and higher average accuracy across zero-shot tasks (59.93\% vs 57.66\%).}

\textcolor{black}{
The grade-level analysis further validates ITPruner's versatility, maintaining strong performance across different educational levels. With 30\% pruning, it achieves 90.68\% accuracy for grades 1-6 and 89.58\% for grades 7-12, demonstrating consistent performance across varying complexity levels. This robust performance across grade levels indicates ITPruner's ability to maintain sophisticated reasoning capabilities while achieving significant parameter reduction.}

\textcolor{black}{
Particularly noteworthy is ITPruner's efficiency in parameter utilization. With a trainable parameter ratio of just 0.97\textperthousand, it achieves performance comparable to or better than models with full parameter counts, while maintaining this efficiency even under different pruning ratios. This characteristic makes ITPruner particularly valuable for practical deployments where computational resources are constrained.}

\textcolor{black}{
These comprehensive results, detailed in Tables \ref{tab:itpruner_on_LLM} and \ref{tab:main_mm}, establish ITPruner as a highly effective approach for vision-language model pruning, offering a compelling combination of strong performance, efficient parameter utilization, and robust cross-modal capabilities. The method's ability to maintain or improve performance while significantly reducing parameter count makes it an attractive solution for efficient deployment of vision-language models in resource-constrained environments.
}

\begin{figure}
\centering
\includegraphics[width=0.95\linewidth]{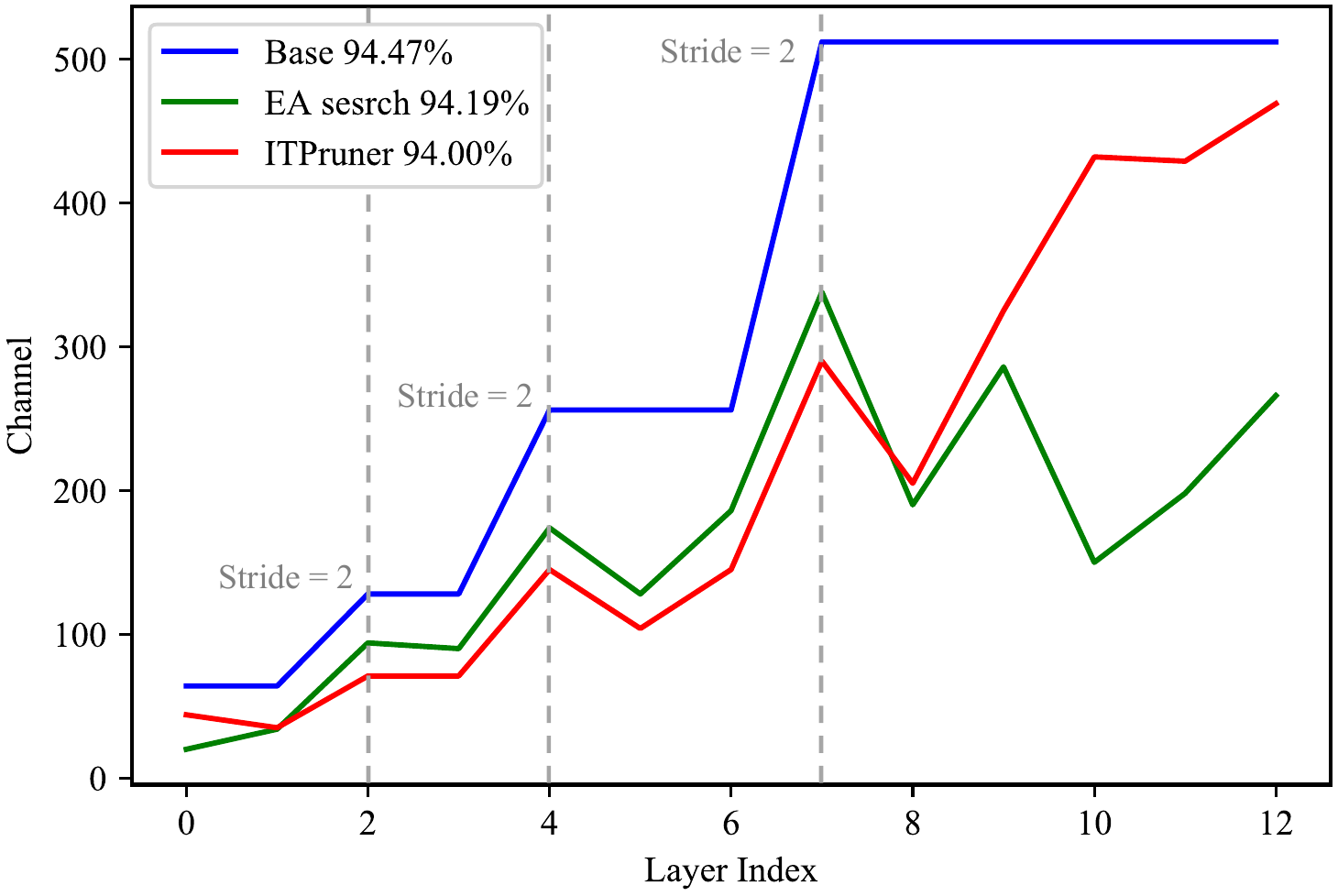}
\caption{The number of channels for the VGG found by ITPruner (red) and evolutionary algorithm (green). The blue line denotes the uncompressed VGG architecture.}
\label{fig:architecture_analysis}
\end{figure}

\subsection{Ablation Study}\label{subsec:abla}

In this section, we analyze the architecture found by ITPruner and the influence of the hyper-parameters $\beta$ and $n$. 
To sum up, the architecture found by ITPruner shows surprisingly similar characteristics compared with the optimal one. 
Meanwhile, ITPruner is easy to use, \emph{i.e.,} the only two hyper-parameters $\beta$ and $n$ show a negligible performance gap with different values. 
The detailed experiments are described as follows. 

\begin{figure} [tb]
\centering
\includegraphics[width=1.0\linewidth]{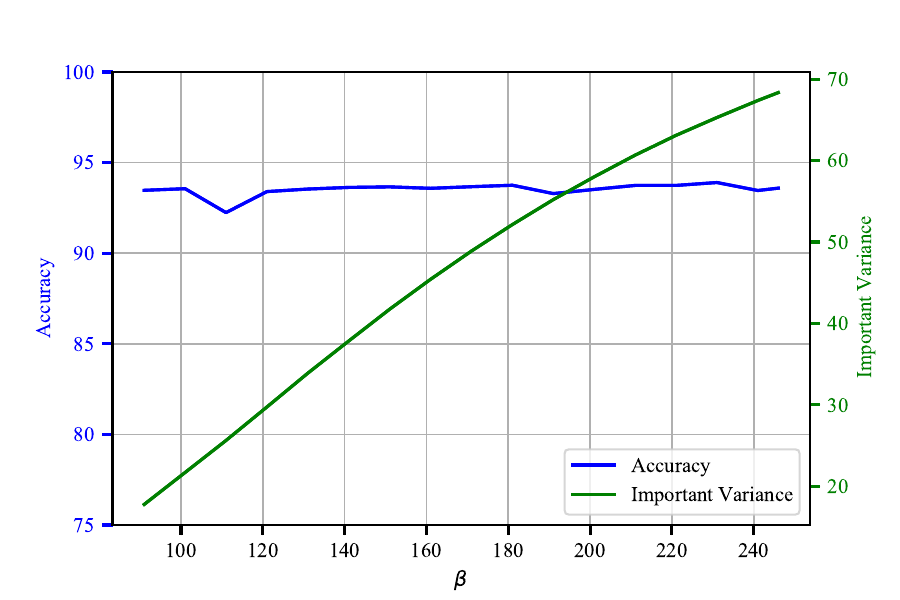}
\caption{Accuracy (blue) and variance (green) of layer-wise importance in different $\beta$.}
\label{fig:beta}
\end{figure}

\begin{figure} [tb]
\centering
\includegraphics[width=1.0\linewidth]{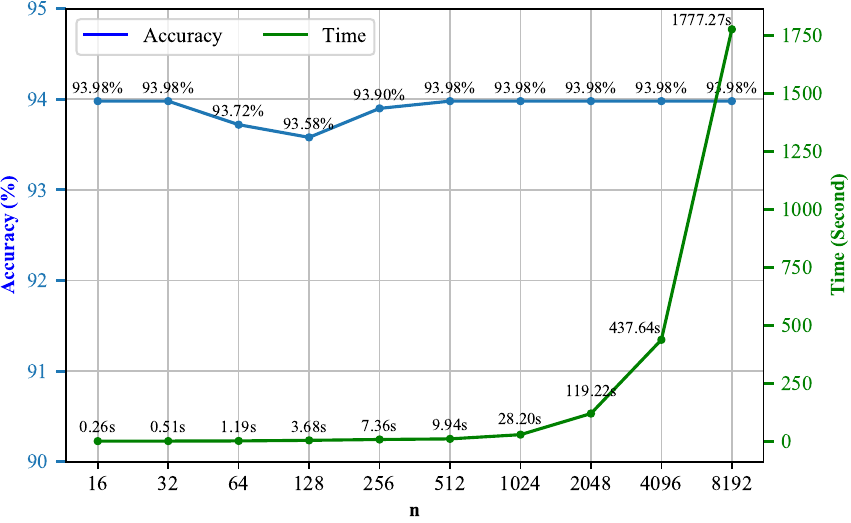}
\caption{{Accuracy (blue) and times costs (green) for calculation $\boldsymbol{H}$ in different sample size $n$.}}
\label{fig:n}
\end{figure}

\textbf{Architecture analysis.} We first adopt Evolutionary Algorithm (EA) \cite{back1996evolutionary} to find the optimal architecture, which is employed as a baseline for a better comparison. 
Specifically, we directly employ the aforementioned training hyper-parameters in the architecture evaluation step. 
Meanwhile, the search algorithm is executed several times to determine the optimal architecture, which is illustrated in the green line of Fig.~\ref{fig:architecture_analysis}. 
Notably, such a process is extremely time-consuming {and} takes a few GPU days in CIFAR-10. 
Interestingly, the proposed ITPruner finds architecture that has similar characteristics compared with the optimal architecture.
In particular, there are significant peaks in these architectures whenever there is a down sampling operation, \emph{i.e.,} stride $= 2$. 
In such an operation, the resolution is halved {and} thus needs to be compensated by more channels. 
Similar observations are also reported in MetaPruning \cite{liu2019metapruning}. 
Nevertheless, we do not need any exhausted search process to obtain such insight. 
We also notice that there is some difference in channel choice between architectures {found} by ITPruner and EA. 
That is, ITPruner is prone to assign more channels in the first and the last layers. 
Meanwhile, the performance gap between these architectures is negligible ($0.19$), which demonstrates the efficiency of our method. 

\begin{table}[tb]
\begin{center}
\caption{\label{tab:object_detection} mean Average precision (mAP), FLOPs and model size of different compression methods on the PASCAL VOC object detection dataset. All the baseline models are trained using our implementation with the same training conditions for a fair \textcolor{black}{comparison}.}
\setlength{\tabcolsep}{1mm}{
\begin{tabular}{ccccc}
\toprule
\multirow{2}{*}{\textbf{Model}}                      & \multirow{2}{*}{\textbf{Method}}          & \textbf{FLOPs} & \textbf{Size} & \textbf{mAP} \\ 
&           &  \textbf{(G)} & \textbf{(Mb)} & \textbf{(0.5)} \\ 
\midrule
\multirow{8}{*}{ResNet-50 \cite{he2016deep}} & Baseline            &   75.37        &     282.39      & 43.76    \\
& Uniform ($\times$ 0.5) & 20.93      & 125.56     & 41.94    \\
& Autoslim \cite{yu2019autoslim}       & 78.41      & 256.98     & 39.66    \\
 & ABCPruner \cite{linchannel}      & 89.02      & 260.56     & 43.04    \\
& TAS \cite{dong2019network}             & 35.56      & 140.23     & 42.02    \\
& FPGM \cite{he2019filter}           & 51.52      & 184.16     & 43.27    \\
& MetaPruning \cite{liu2019metapruning}    & 26.02      & 151.88     & 42.65    \\
\cmidrule{2-5} 
& \textbf{ITPruner}       & \textbf{23.48}      & \textbf{49.59}     & \textbf{44.65}    \\ 
\bottomrule
\end{tabular}}
\end{center}
\end{table}

\begin{table*}[th]
\color{black}
\caption{ITPruner MSCOCO and Flickr30K datasets zero-shot image-text retrieval results. $\dagger$ denotes the results in~\cite{kim2023Misalign, Lee2022Uniclip, Yang2023Alip}.}
\label{itpruner_clip}
\resizebox{1.0\textwidth}{!}{%
\begin{tabular}{l|c|c|cccccc|cccccc}
\midrule
Method & Vision Encoder & Text Encoder & \multicolumn{6}{c|}{MSCOCO (5K test set)} & \multicolumn{6}{c}{Flickr30K (1K test set)} \\
 & Width Depth & Width Depth & TR @1 & TR @5 & TR @10 & IR @1 & IR @5 & IR @10 & TR @1 & TR @5 & TR @10 & IR @1 & IR @5 & IR @10 \\
\midrule
\multicolumn{14}{l}{$Pre-trained~on~WIT-400M$} \\
CLIP-ViT-L/14~\cite{Radford2021Learning} & 1024 24 & 768 12 & 56.3 & 79.4 & 86.6 & 36.5 & 61.1 & 71.2 & 85.2 & 97.5 & 99.1 & 64.9 & 87.3 & 92.2 \\
CLIP-ViT-B/32~\cite{Radford2021Learning} & 768 12 & 512 12 & 50.1 & 75.0 & 83.5 & 30.5 & 56.0 & 66.9 & 78.8 & 94.9 & 98.2 & 58.8 & 93.6 & 90.2 \\
\midrule
\multicolumn{14}{l}{$Pre-trained~on~CC3M$} \\
EfficientVLM~\cite{Wang2021Efficientvlm} & 1024 12 & 768 6 & 46.6 & 71.7 & 81.3 & 35.9 & 61.6 & 71.8 & 78.8 & 94.9 & 98.2 & 58.8 & 93.6 & 90.2 \\
TinyCLIP~\cite{Wu2023Tinyclip} & 512 24 & 768 6 & 52.7 & 76.5 & 84.8 & 36.6 & 63.0 & 73.6 & 80.5 & 96.3 & 98.5 & 66.3 & 89.1 & 93.7 \\
$\text{MoPE-CLIP}_{large}$~\cite{Lin2024Mope} & 512 24 & 384 12 & 58.0 & 81.6 & 88.5 & 40.6 & 66.0 & 75.5 & 86.5 & 97.7 & 99.0 & 69.8 & \textbf{90.6} & 95.3 \\
ITPruner & 512 24 & 384 12 & \textbf{58.4} & \textbf{81.8} & \textbf{88.9} & \textbf{41.4} & \textbf{66.1} & \textbf{76.1} & \textbf{86.7} & \textbf{97.8} & \textbf{99.5} & \textbf{70.0} & 90.3 & \textbf{95.3} \\
$\text{DynaCLIP}_{base}$~\cite{Hou2020Dynabert} & 384 18 & 384 12 & 51.3 & 75.5 & 84.6 & 35.8 & 61.8 & 72.6 & 79.8 & 96.1 & 98.2 & 64.6 & 87.8 & 93.1 \\
$\text{DynaCLIP}_{small}$~\cite{Hou2020Dynabert} & 384 18 & 192 12 & 46.7 & 72.7 & 92.2 & 33.2 & 59.5 & 70.3 & 75.9 & 94.6 & 98.3 & 60.9 & 86.1 & 91.9 \\
$\text{MoPE-CLIP}_{base}$~\cite{Lin2024Mope} & 384 18 & 384 12 & 52.8 & 78.1 & 86.0 & 37.3 & 63.5 & 73.6 & 82.8 & 97.1 & 98.8 & 66.7 & 88.7 & 94.1 \\
ITPruner & 384 18 & 384 12 & 53.1 & 78.7 & 86.1 & 38.2 & 63.8 & 74.0 & 83.0 & 97.1 & 98.5 & 66.6 & 89.0 & 94.3 \\
$\text{MoPE-CLIP}_{small}$~\cite{Lin2024Mope} & 384 18 & 192 12 & 50.3 & 75.9 & 84.8 & 35.6 & 61.7 & 72.2 & 80.2 & 95.6 & 98.5 & 64.7 & 87.8 & 93.0 \\
ITPruner & 384 18 & 192 12 & 50.4 & 76.3 & 85.2 & 36.1 & 62.0 & 73.1 & 80.8 & 95.7 & 98.5 & 65.0 & 87.9 & 93.3 \\
\midrule
\multicolumn{14}{l}{$Pre-trained~on~YFCC15M$} \\
CLIP-ViT-B/32$^\dagger$~\cite{Radford2021Learning} & 768 12 & 512 12 & 20.8 & 43.9 & 55.7 & 13.0 & 31.7 & 42.7 & 34.9 & 63.9 & 75.9 & 23.4 & 47.2 & 58.9 \\
SLIP-ViT-B/32$^\dagger$~\cite{Mu2022Slip} & 768 12 & 512 12 & 27.7 & 52.6 & 63.9 & 18.2 & 39.2 & 51.0 & 47.8 & 76.5 & 85.9 & 32.3 & 58.7 & 68.8 \\
DeCLIP-ViT-B/32$^\dagger$~\cite{Li2021Supervision} & 768 12 & 512 12 & 28.3 & 53.2 & 64.5 & 18.4 & 39.6 & 51.4 & 51.4 & 80.2 & 88.9 & 34.3 & 60.3 & 70.7 \\
UniCLIP-ViT-B/32$^\dagger$~\cite{Lee2022Uniclip} & 768 12 & 512 12 & 32.0 & 57.7 & 69.2 & 20.2 & 43.2 & 54.4 & 52.3 & 81.6 & 89.0 & 34.8 & 62.0 & 72.0 \\
MCD-ViT-B/32$^\dagger$~\cite{kim2023Misalign} & 768 12 & 512 12 & 32.2 & 58.7 & 71.2 & 20.7 & 43.5 & 55.3 & 57.6 & 82.6 & 91.1 & 36.4 & 64.8 & 74.1 \\
ALIP-ViT-B/32$^\dagger$~\cite{Yang2023Alip} & 768 12 & 512 12 & 46.8 & 72.4 & 81.8 & 29.3 & 54.4 & 65.4 & 70.5 & 91.9 & 95.7 & 48.9 & 75.1 & 82.9 \\
$\text{MoPE-CLIP}_{base}$~\cite{Lin2024Mope} & 384 18 & 384 12 & 55.6 & 78.6 & 86.1 & \textbf{37.1} & 63.1 & 73.5 & 86.1 & 97.9 & 99.6 & 66.4 & 89.2 & 94.2 \\
ITPruner & 384 18 & 384 12 & \textbf{56.0} & \textbf{78.8} & \textbf{86.2} & 37.0 & \textbf{63.4} & \textbf{73.9} & \textbf{86.3} & \textbf{98.1} & \textbf{99.6} & \textbf{66.9} & \textbf{89.5} & \textbf{94.8} \\
\midrule
\end{tabular}
}
\end{table*}

\begin{table*}[th]
\color{black}
\caption{Results of ITPruner on ImageNet-1K and ADE20K datasets. All models undergo finetuning exclusively on ImageNet-1K. Methods marked with ``Extra Pretrained Teacher'' require a pre-trained model and cannot be trained from scratch. $\ddagger$ indicates that the model uses an MAE pre-trained model prior to sparse pre-training.}
\label{itpruner_mae}
\resizebox{1.0\textwidth}{!}{%
\begin{tabular}{l|c|c|c|c|c|c}
\midrule
Method & Params & Searching & Encoder & Extra Pretrained/ & Classification & Segmentation \\
 & (M) & epochs & ratio & Teacher & Top-1 Acc (\%) & mIoU (\%) \\
\midrule
DeiT~\cite{touvron2021training} & 5.8 & - & - & Label & 72.2 & 38.0 \\
\midrule
\multicolumn{7}{l}{$Tiny-scale$} \\
MAE~\cite{He2022Masked} & 5.8 & 1600 & 25\% & \bxmark & 71.6 & 37.6 \\
MoCo~\cite{Chen2021moco} & 5.8 & 300 & 100\% & EMA & 73.3 & 39.3 \\
TinyMIM~\cite{Ren2023Tinymim} & 5.8 & 300 & 100\% & TinyMIM-S & 75.8 & 44.0 \\
G2SD~\cite{Huang2023Generic} & 5.8 & 300 & 100\% & ViT-B & 76.3 & 41.4 \\
MAE-lite~\cite{Wang2023light} & 5.8 & 400 & 25\% & ViT-B & 76.5 & 43.9 \\
SparseMAE~\cite{Zhou2023Sparsemae} & 5.8 & 400 & 25\% & \bxmark & 80.5 & 45.2 \\
ITPruner & 5.8 & 0 & 25\% & \bxmark & 81.6 & 45.7 \\
\midrule
DeiT~\cite{touvron2021training} & 22 & - & - & Label & 79.9 & \\
\midrule
\multicolumn{7}{l}{$Small-scale$} \\
MAE~\cite{He2022Masked} & 22 & 1600 & 25\% & \bxmark & 80.6 & 42.8 \\
MoCo~\cite{Chen2021moco} & 22 & 300 & 100\% & EMA & 81.4 & 43.9 \\
DINO~\cite{Caron2021Emerging} & 22 & 300 & 100\% & EMA & 81.5 & 45.3 \\
CAE~\cite{Chen2024Context} & 22 & 300 & 100\% & DALL-E & 82.0 & - \\
TinyMIM~\cite{Ren2023Tinymim} & 22 & 300 & 100\% & TinyMIM-ViT-B & 83.0 & 48.4 \\
G2SD~\cite{Huang2023Generic} & 22 & 300 & 25\% & ViT-B & 82.0 & 46.2 \\
SparseMAE~\cite{Zhou2023Sparsemae} & 11.3 & 300 & 25\% & ViT-B$^\ddagger$ & 83.2 & 48.4 \\
SparseMAE~\cite{Zhou2023Sparsemae} & 11.3 & 400 & 25\% & \bxmark & 82.1 & 46.7 \\
ITPruner & 11.3 & 0 & 25\% & \bxmark & 82.4 & 46.9 \\
\midrule
\end{tabular}
}
\end{table*}

\begin{table*}[th]
\color{black}
\centering
\caption{Efficiency improvement results using ITPruner on Deformable DETR-based architectures. The comparison includes both efficient CNN-based models and other efficient DETR variants. ResNet-50 and Swin-T (pre-trained on ImageNet-1K) serve as backbones for all models, with the exception of EfficientDet and YOLO series.}
\label{itpruner_dino}
\resizebox{1.0\textwidth}{!}{%
\begin{tabular}{l|c|cccccc|c|c}
\midrule
Model & \#epochs & AP & AP$_{50}$ & AP$_{75}$ & AP$_S$ & AP$_M$ & AP$_L$ & GFLOPs & Encoder \\
& & & & & & & & & GFLOPs \\
\midrule
EfficientDet-D6~\cite{Tan2020Efficientdet} & - & 51.3 & - & - & - & - & - & 226 & - \\
YOLOv5-X~\cite{Glennyolov5} & - & 50.7 & - & - & - & - & - & 206 & - \\
YOLOv7-X~\cite{Wang2023yolov7} & - & 52.9 & - & - & - & - & - & 190 & - \\
\midrule
\multicolumn{10}{l}{Swin-T backbone} \\
VIDT+~\cite{Song2022VIDT} & 50 & 49.7 & 67.7 & 54.2 & 31.6 & 53.4 & 65.9 & - & - \\
D2ETR~\cite{Lin2022D2ETR} & 50 & 49.1 & - & - & - & - & - & 127 & - \\
\midrule
DINO~\cite{Caron2021Emerging} & 36 & 54.1 & 72.0 & 59.3 & 38.3 & 57.3 & 68.6 & 243 & 137 \\
Lite-DINO H3L1-(2+1)x3~\cite{Li2023Litedetr} & 36 & 53.9 & 72.0 & 58.8 & 37.9 & 57.0 & 69.1 & 159 & 53 \\
ITPruner & 0 & 54.0 & 72.0 & 59.0 & 38.0 & 57.0 & 69.2 & 142 & 36 \\
\midrule
H-DETR~\cite{Jia2023Detrs} & 36 & 53.2 & 71.5 & 58.2 & 35.9 & 56.4 & 68.2 & 234 & 137 \\
Lite-H-DETR H3L1-(2+1)x3~\cite{Li2023Litedetr} & 36 & 53.0 & 71.3 & 58.2 & 36.3 & 56.3 & 68.1 & 152 & 53 \\
ITPruner & 0 & 53.1 & 71.3 & 58.4 & 36.4 & 56.4 & 68.2 & 140 & 41 \\
\midrule
\multicolumn{10}{l}{ResNet-50 backbone} \\
DFFT~\cite{Chen2022DFFT} & 36 & 46.0 & - & - & - & - & - & 101 & 18 \\
PnP-DETR~\cite{Wang2021Pnp-detr} & 36 & 43.1 & 63.4 & 45.3 & 22.7 & 46.5 & 61.1 & 104 & 29 \\
AdaMixer~\cite{Gao2022Adamixer} & 36 & 47.0 & 66.0 & 51.1 & 30.1 & 50.2 & 61.8 & 132 & - \\
IMFA-DETR~\cite{Zhang2023Towards} & 36 & 45.5 & 45.0 & 49.3 & 27.3 & 48.3 & 61.6 & 108 & $\approx$20 \\
\midrule
DINO~\cite{Caron2021Emerging} & 36 & 50.7 & 68.6 & 55.4 & 33.5 & 54.0 & 64.8 & 235 & 137 \\
Lite-DINO H3L1-(2+1)x3~\cite{Li2023Litedetr} & 36 & 50.4 & 68.5 & 54.6 & 33.5 & 53.6 & 65.5 & 151 & 53 \\
ITPruner & 0 & 50.3 & 68.3 & 54.5 & 33.4 & 53.2 & 65.1 & 135 & 37 \\
\midrule
H-DETR~\cite{Jia2023Detrs} & 36 & 50.0 & 68.3 & 54.4 & 32.9 & 52.7 & 65.3 & 226 & 137 \\
Lite-H-DETR H3L1-(2+1)x3~\cite{Li2023Litedetr} & 36 & 49.5 & 67.6 & 53.9 & 32.0 & 52.8 & 64.0 & 142 & 53 \\
ITPruner & 0 & 49.3 & 67.3 & 53.5 & 32.1 & 52.9 & 63.8 & 133 & 44 \\
\midrule
\end{tabular}
}
\end{table*}

\begin{table*}[th]
\color{black}
    \centering
    \caption{Accuracy and Perplexity Evaluation of ITPruner on LLaMA2-7B Model with the wikitext2 Dataset and Seven Zero-Shot Tasks.}
    \label{tab:itpruner_on_LLM}
    \resizebox{\textwidth}{!}{%
    \begin{tabular}{@{}lccccccccc@{}}
    \toprule
    \textbf{Model} & \textbf{WikiText2} & \textbf{HellaSwag} & \textbf{Winogrande} & \textbf{BoolQ} & \textbf{OBQA} & \textbf{PIQA} & \textbf{ARC-e} & \textbf{ARC-c} & \textbf{Mean} \\ 
    \midrule
    Dense & 5.12 & 75.97 & 69.06 & 77.74 & 44.20 & 78.07 & 76.35 & 46.33 & 66.82 \\ 
    \midrule
    \multicolumn{10}{l}{\textit{20\% Structured Pruning}} \\ 
    LLM-Pruner~\cite{Ma2023Llm-pruner} & 10.48 & 68.72 & 63.54 & 65.14 & 39.80 & 75.90 & 68.81 & 39.42 & 60.19 \\ 
    Wanda-sp~\cite{Sun2023wanda} & 12.01 & 72.86 & 65.98 & 67.61 & 38.80 & 76.61 & 72.18 & 42.24 & 62.33 \\ 
    Wanda-sp + ITPruner & \textbf{9.92} & \textbf{73.11} & \textbf{67.98} & \textbf{69.63} & \textbf{39.90} & \textbf{76.93} & \textbf{72.19} & \textbf{43.60} & \textbf{63.33} \\ 
    \midrule
    \multicolumn{10}{l}{\textit{30\% Structured Pruning}} \\ 
    LLM-Pruner~\cite{Ma2023Llm-pruner} & 17.90 & 57.77 & 55.49 & 50.12 & 36.80 & 71.87 & 58.84 & 32.34 & 51.89 \\ 
    Wanda-sp~\cite{Sun2023wanda} & 24.53 & 65.94 & 58.33 & 63.12 & 38.00 & 74.76 & 65.99 & 37.46 & 57.66 \\ 
    Wanda-sp + ITPruner & \textbf{16.25} & \textbf{68.63} & \textbf{63.40} & \textbf{64.71} & \textbf{40.20} & \textbf{75.30} & \textbf{67.85} & \textbf{39.42} & \textbf{59.93} \\ 
    \midrule
    \multicolumn{10}{l}{\textit{50\% Unstructured Pruning}} \\ 
    Wanda~\cite{Sun2023wanda} & 6.42 & 70.84 & 68.19 & 75.99 & 41.80 & 76.00 & 72.77 & 42.41 & 64.00 \\ 
    Wanda + ITPruner & 6.31 & 72.34 & 68.50 & 76.31 & 43.06 & \textbf{76.59} & 73.12 & 43.02 & 64.71 \\ 
    sparseGPT~\cite{Frantar2023Sparsegpt} & 6.51 & 71.37 & 69.85 & 75.02 & 40.80 & 75.46 & 73.27 & 41.55 & 63.90 \\ 
    sparseGPT + ITPruner & \textbf{6.33} & \textbf{72.54} & \textbf{69.86} & \textbf{75.43} & \textbf{42.09} & 76.51 & \textbf{73.89} & \textbf{43.16} & \textbf{64.78} \\ 
    \midrule
    \multicolumn{10}{l}{\textit{60\% Unstructured Pruning}} \\ 
    Wanda~\cite{Sun2023wanda} & 10.79 & 59.25 & 64.88 & 65.84 & 38.60 & 72.09 & 64.56 & 33.02 & 56.89 \\ 
    Wanda + ITPruner & 9.32 & \textbf{67.09} & 65.72 & 68.34 & 40.50 & \textbf{74.47} & 67.36 & 38.78 & 60.32 \\ 
    sparseGPT~\cite{Frantar2023Sparsegpt} & 10.14 & 62.32 & 65.90 & 71.99 & 37.60 & 71.11 & 64.02 & 35.15 & 58.30 \\ 
    sparseGPT + ITPruner & \textbf{9.13} & {66.78} & \textbf{65.94} & \textbf{73.79} & \textbf{40.60} & {73.99} & \textbf{67.82} & \textbf{39.78} & \textbf{61.24} \\ 
    \bottomrule
    \end{tabular}%
    }
\end{table*}

\begin{table*}[th]
\color{black}
    \centering
    \caption{ITPruner sparsity comparison on ScienceQA \textit{test} set.  Question classes: NAT = natural science, SOC = social science, LAN = language science, TXT = text context, IMG = image context, NO = no context, G1-6 = grades 1-6, G7-12 = grades 7-12. \#T-Ratio denotes the ratio of trainable parameters.}
    \resizebox{1.0\linewidth}{!}{
        \begin{tabular}{lccc|ccc|ccc|cc|c}
            \toprule
            \multirow{2}{*}{Method} & \multirow{2}{*}{\#T.Ratio} & \multirow{2}{*}{Sparsity} & \multirow{2}{*}{LLM} & \multicolumn{3}{c|}{Subject} & \multicolumn{3}{c|}{Context Modality} & \multicolumn{2}{c|}{Grade} & \multirow{2}{*}{Average} \\
            &&&& NAT & SOC & LAN & TXT & IMG & NO & G1-6 & G7-12 & \\
            \midrule
            \multicolumn{13}{l}{\it Zero- $\&$ few-shot methods} \\
            Human~\cite{scienceqa} &-& - & \bxmark & 90.23 & 84.97 & 87.48 & 89.60 & 87.50 & 88.10 & 91.59 & 82.42 & 88.40 \\
            GPT-3.5~\cite{scienceqa}&-& - & \bcmark & 74.64 & 69.74 & 76.00 & 74.44 & 67.28 & 77.42 & 76.80 & 68.89 & 73.97 \\
            GPT-3.5 (CoT)~\cite{scienceqa} &-& - & \bcmark & 75.44 & 70.87 & 78.09 & 74.68 & 67.43 & 79.93 & 78.23 & 69.68 & 75.17 \\
            GPT-4~\cite{gpt4} &-& - & \bcmark & 84.06 & 73.45 & 87.36 & 81.87 & 70.75 & 90.73 & 84.69 & 79.10 & 82.69 \\ \midrule
            \multicolumn{10}{l}{\it Representative \& PEFT models}\\
            UnifiedQA~\cite{scienceqa}&100\% & 0\% & \bxmark & 71.00 & 76.04 & 78.91 & 66.42 & 66.53 & 81.81 & 77.06 & 68.82 & 74.11 \\
            MM-CoT$_{Base}$~\cite{zhang2023multimodal} &100\% & 0\% & \bxmark & 87.52 & 77.17 & 85.82 & 87.88 & 82.90 & 86.83 & 84.65 & 85.37 & 84.91 \\
            MM-CoT$_{Large}$~\cite{zhang2023multimodal}& 100\% & 0\% & \bxmark & 95.91 & 82.00 & 90.82 & 95.26 & 88.80 & 92.89 & 92.44 & 90.31 & 91.68 \\
            LLaVA~\cite{liu2023visual} &100\% & 0\% & \bcmark & 90.36 & 95.95 & 88.00 & 89.49 & 88.00 & 90.66 & 90.93 & 90.90 & 90.92 \\  
            LaVIN-7B-MMA~\cite{luo2023cheap} &0.56\textperthousand & 0\% & \bcmark & 89.25 & 94.94 & 85.24 & 88.51 & 87.46 & 88.08 & 90.16 & 88.07 & 89.41 \\                        
            \midrule
            \multicolumn{10}{l}{\it Parameter-efficient pruning methods}\\ 
            LaVIN-7B-MaP &0.97 \textperthousand & 30\% & \bcmark & 86.23 & 91.62 & 81.45 & 85.18 & 85.55 & 84.36 & 86.67 & 85.71 & 85.85 \\
            LaVIN-7B-ITPruner (Ours) &0.97\textperthousand & 30\% & \bcmark & \textbf{90.05} & \underline{94.83} & \textbf{87.09} & \textbf{89.98} & \textbf{89.24} & \textbf{88.64} & \textbf{90.68} & \textbf{89.58} & \textbf{90.29} \\
            LaVIN-7B-MaP &0.97 \textperthousand & 50\% & \bcmark & 85.65 & 91.28 & 81.14 & 84.83 & 84.99 & 83.75 & 86.14 & 84.89 & 85.33 \\
            LaVIN-7B-ITPruner (Ours) &0.97\textperthousand  & 50\% & \bcmark & \underline{89.43} & 94.04 & \underline{85.55} & \underline{89.20} & \underline{88.65} & 87.18 & 89.65 & \underline{88.93} & \underline{89.39} \\
            \bottomrule
        \end{tabular}
    }
    \vspace{2mm}
    \label{tab:main_mm}
    \vspace{-2em}
\end{table*}

\textbf{Influence of $\beta$ and $n$.} 
There are only two hyper-parameters in ITPruner. 
One is the importance factor $\beta$, which is proposed to control relative compression rates between different layers. 
The other one is sample image $n$, which determines the computation cost of ITPruner. 
Fig.~\ref{fig:beta} reports the accuracy and variance of layer-wise importance in different $\beta$. 
As we can see, layer-wise importance $\boldsymbol{i}$ tends to {have} a significant variance with a large $\beta$, and vice versa. 
Such variance finally determines the difference of layer-wise compression rate through Eq.~\ref{eq:HIS_prob}. 
However, the performance gap between different $\beta$ is negligible. 
Similar results are observed for different $n$. 
In particular, 
we further conduct {an} experiment to demonstrate the influence of the sample size $n$, which is reported in Fig.~\ref{fig:n}. 
As we can see, the variance of the accuracy with different $n$ is {within} $0.1\%$, which means accuracy will not fluctuate greatly with the change of $n$.
{Meanwhile, when $n>512$, we obtain exactly the same network architecture. 
In other words, although the time costs to obtain $\boldsymbol{H}$ exponentially increase with $n$, we set $n=1, 024$ to obtain exactly the same architecture. 
Overall, We conclude that ITPruner is robust to the hyper-parameters.}

\begin{figure} [htb]
\centering
\includegraphics[width=0.85\linewidth]{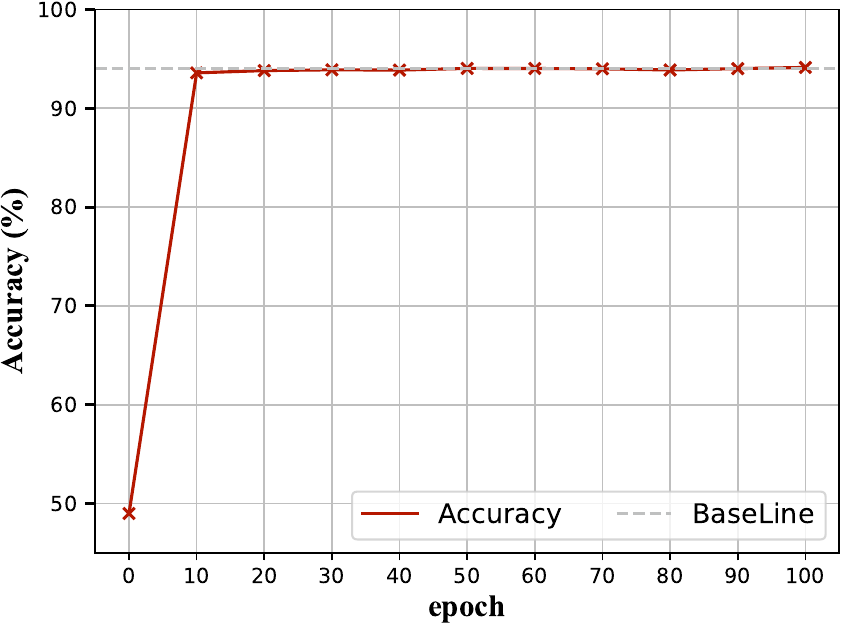}
\caption{Accuracy of ITPruner using the pre-trained model with different training epochs. }
\label{fig:fewer_epoch}
\end{figure}

\textcolor{black}{
\textbf{Different training epochs in pre-training.} Compared to the Growing \cite{yuan2020growing} based method, ITPruner need an extra pre-training stage. 
Through extensive experiments, we find that the converged pre-trained model is actually unnecessary in ITPruner. 
Specifically, we evaluate ITPruner on pre-trained models with different training epochs. As reported in Fig.~\ref{fig:fewer_epoch}, models trained using $>10\%$ epochs show a negligible performance gap. 
In other words, when the training epoch is set to be $10$ in pre-training stage, ITPruner shows a better performance ($93.43$ \text{vs} $92.5$) with a similar cost saving ($5.23\times$ \text{vs} $4.95\times$) in VGG compression. 
Therefore, it is clear that ITPruner does achieve lower training cost saving with a clear performance gap. The detailed answers to the comments are listed as follows accordingly. 
}

\begin{table}[tb]
\begin{center}
\caption{\label{tab:semantic_segmentation} mIOU, FLOPs and model size of different compression methods on the PASCAL VOC semantic segmentation dataset. All the baseline models are trained using our implementation with the same training conditions for a fair comparision.}
\setlength{\tabcolsep}{1mm}{
\begin{tabular}{ccccc}
\toprule
\multirow{2}{*}{\textbf{Model}}                      & \multirow{2}{*}{\textbf{Method}}          & \textbf{FLOPs} & \textbf{Size} & \multirow{2}{*}{\textbf{mIOU}} \\ 
                      &           &  \textbf{(G)} & \textbf{(Mb)} & \\ 
\midrule
\multirow{6}{*}{ResNet-50 \cite{he2016deep}} & Baseline            &   4.65        &     140.03      & 55.0    \\
                           & Autoslim \cite{yu2019autoslim}        & 2.62      & 141.06     & 51.7    \\
                           & TAS \cite{dong2019network}            & 2.41      & 106.07     & 51.0    \\
                           & FPGM \cite{he2019filter}           & 2.41      & 90.85     & 48.8    \\
                           & MetaPruning \cite{liu2019metapruning}    & 2.57      & 122.48     & 53.5    \\
                           \cmidrule{2-5} 
                           & \textbf{ITPruner}       & \textbf{2.46}      & \textbf{95.44}     & \textbf{54.5}    \\ 
                            \bottomrule
\end{tabular}}
\end{center}

\end{table}

\subsection{Transferability}\label{subsec:trans}

{
To further test the transferability of ITPruner on different tasks in the wild, we conduct extensive experiments on detection and semantic segmentation. Specifically, we compress the backbones on the corresponding tasks with different compression methods to validate the generalizability of ITPruner. On both of these tasks, we obtain superior performance compared to other baseline methods, which reveals that our method can be well transferred to other tasks. }

{
\textbf{Compression on Object Detection}
We first apply ITPruner to object detection, which is a fundamental task and deals with detecting objects of a certain class. Specifically, we directly employ ITPruner to compress the backbone in Faster-RCNN~\cite{ren2016faster}, which is a widely-used lightweight object detection framework. 
All the experiments are evaluated on the PASCAL VOC dataset~\cite{everingham2010pascal}, VOC2012 \textit{trainval} and VOC2007 \textit{trainval} are used for training, and VOC2007 \textit{test} are used for testing. 
Tab.~\ref{tab:object_detection} reports the corresponding results on VOC2007 \textit{test}. Obviously, ITPruner clearly achieves a superior trade-off: The model found by ITPruner outperforms baseline with a $3.2\times$ and $5.69\times$ compression in terms of FLOPs and model size, respectively; ITPruner outperforms Autoslim, ABCPruner and TAS by $4.99$, $1.61$ and $2.03$ respectively with much less model size. }

{
\textbf{Compression on Semantic Segmentation}
We further validate the effectiveness of ITPruner on semantic segmentation tasks. In particular, semantic segmentation aims to provide a dense labeling map of object categories for a given image. Similar to object detection, the networks used in semantic segmentation are divided into backbone and head. We directly use ITPruner to compress the backbone in DeepLabV3~\cite{chen2017rethinking}. All the experiments are evaluated on the PASCAL VOC2012 \cite{everingham2010pascal}, which contains high-resolution with pixel-wise annotations. The dataset includes $2, 913$ finely annotated images collected 
with 20 object categories, and is split with $1,464$ for training, $1,449$ for validation. 
Tab.~\ref{tab:semantic_segmentation} shows the results on VOC2012.  Specifically, the model compressed by ITPruner significantly outperforms automatic compression methods Autoslim, TAS, and MetaPruning by $2.8$, $3.5$ and $1.0$, respectively. At the same time, the FLOPs is reduced to $2.46$G, which is similar or much less to the other baselines. We thus conclude that ITPruner achieves a superior trade-off between mIOU and model complexity. Compared with the local metric based method FPGM, our model outperforms it with $5.7$. Compared with the redundant original model Resnet-50, the performance of our compact model drops slightly by $0.5$, but the latter of which has a much lower parameter size and FLOPs. 
}



\section{Conclusion}

In this paper, we present an information theory-inspired strategy for {automated} network pruning (ITPruner), which determines the importance of layer by observing the independence of feature maps and \textcolor{black}{does} not need any search process. 
To that effect, ITPruner is derived from the information bottleneck principle, which proposed to use normalized HSIC to measure and determine the layer-wise compression ratio. 
We also mathematically prove the robustness of the proposed method as well as the relation to mutual information. 
Extensive experiments on various modern CNNs demonstrate the effectiveness of ITPruner in reducing the computational complexity and model size. 
In future work, we will do more on the theoretical analysis on neural architecture search to effectively find optimal architecture without any search process.


\begin{appendices}



\section{Proof of Theorem 1}
\begin{theorem} \label{theorem:compression_bottle}
    Assuming $X$ is the input random variable follows a Markov random field structure and the Markov random field is ergodic. 
    For a network that have $L$ hidden representations $X_1, X_2, ..., X_L$, with a probability $1-\delta$, the generalization error $\epsilon$ is bounded by
    \begin{equation}\nonumber
        \epsilon \leq \sum_{i=1}^L\sqrt{\frac{\log \frac{2}{\delta} + \log2\text{I}(X ;X_i)}{2n}},
    \end{equation}
    where $n$ is the number of training examples. 
\end{theorem}

\begin{proof}
Let $\mathcal{D}_n = {(x_1, y_1),...,(x_n, y_n)}$ be an i.i.d  random sample from $P_{X,Y}$. We define the loss function as a mapping $l:\mathcal{Y}\times\mathcal{Y}\rightarrow\mathbb{R}^+$ and the performance of a predictor $h:\mathcal{X}\rightarrow\mathcal{Y}$ is measured by the expectation on $P_{X,Y}$, which is also known as generalization error and formally defined as $R(h):=\mathbb{E}_{X,Y}\left[l(Y, h(X))\right].$ In practice, $P_{X,Y}$ is usually unknown, we thus estimate $\hat{h}$ based on $\mathcal{D}_n$ and the corresponding empirical risk is defined as $\hat{R}_n(h):=\frac{1}{n}\sum_{i=1}^nl(y_i, h(x_i)).$ For simplicity, we further define the mapping function $l = I\left(h(X)\neq Y\right)$, where $I(.,.)$ is the indicator function. In this case, applying the Hoeffding inequality \cite{hoeffding1994probability} we have 
\begin{equation}
    P\left(\left|\hat{R}_n(h)-R(h)\right|\geq \epsilon \right)\leq 2\exp\left(-2n\epsilon^2\right).
\end{equation}
Together with PAC learning theory \cite{mohri2018foundations}, we obtain the following corollary:
\begin{corollary}\label{coro:pac_hoeffding}
Let $\mathcal{H}$ be a finite set of hyothesis. Then for all $\epsilon>0$, we have 
\begin{equation}
    P\left(\exists h\in \mathcal{H} \middle| \left|\hat{R}_n(h)-R(h)\right|\geq \epsilon \right)\leq 2\left| \mathcal{H} \right|\exp\left(-2n\epsilon^2\right).
\end{equation}
\end{corollary}
We further control the probability in Corollary \ref{coro:pac_hoeffding} with a confidence $\delta$. That is, we ask $$2\left| \mathcal{H} \right|\exp\left(-2n\epsilon^2\right)\leq\delta.$$
In other words, with a probability $\delta$, we have
\begin{align*}
    2\left| \mathcal{H} \right|\exp\left(-2n\epsilon^2\right)&\leq\delta\\
    \exp\left(-2n\epsilon^2\right)&\leq\frac{\delta}{2\left| \mathcal{H} \right|}\\
    -2n\epsilon^2&\leq \log\frac{\delta}{2\left| \mathcal{H} \right|}\\
    \epsilon^2&\geq \frac{\log\frac{2}{\delta}+ \log\left| \mathcal{H} \right|}{2n}.
\end{align*}
Meanwhile, we can also conclude that $\epsilon^2$ is also bounded
\begin{equation}
    \epsilon^2\leq \frac{\log\frac{2}{\delta}+ \log\left| \mathcal{H} \right|}{2n},
\end{equation}
with probability $1-\delta$. 
To prove our theorem, we further introduce a lemma \cite{shwartz2018representation} about Asymptotic Equipartition Property (AEP) \cite{cover1999elements} as follows:
\setcounter{theorem}{0}
\begin{lemma}
Assuming $X$ is a random variable follows a Markov random field structure and the Markov random field is ergodic. Then, we have that 
\begin{equation}
    \left| \mathcal{H} \right| \leq 2^{H(X)},
\end{equation}
Let $X_i$ be a mapping of $X$, the size of typical set $|\mathcal{T}|$ is bounded by 
\begin{equation}
    |\mathcal{T}| \leq \frac{2^{H(X)}}{2^{H(X|X_i)}} = 2^{\text{I}(X;X_i)}
\end{equation}
\end{lemma}
Therefore, we conclude that for a mapping $T$, the square error is bounded by
\begin{equation}
    \epsilon^2_T\leq \frac{\log\frac{2}{\delta}+ \log|\mathcal{T}|}{2n} \leq \frac{\log\frac{2}{\delta}+ \text{I}(X;X_i)\log2}{2n}.
\end{equation}
Considering that there are $L$ mappings in the network, we thus conclude that the overall square error is bounded by 
\begin{equation}
    \epsilon \leq \sum_{i=1}^L\sqrt{\frac{\log \frac{2}{\delta} + \text{I}(X;X_i)\log2}{2n}},
\end{equation}
 with a probability $1-\delta$.
\end{proof}

\color{black}

\begin{figure} [tb]
\centering
\includegraphics[width=1.0\linewidth]{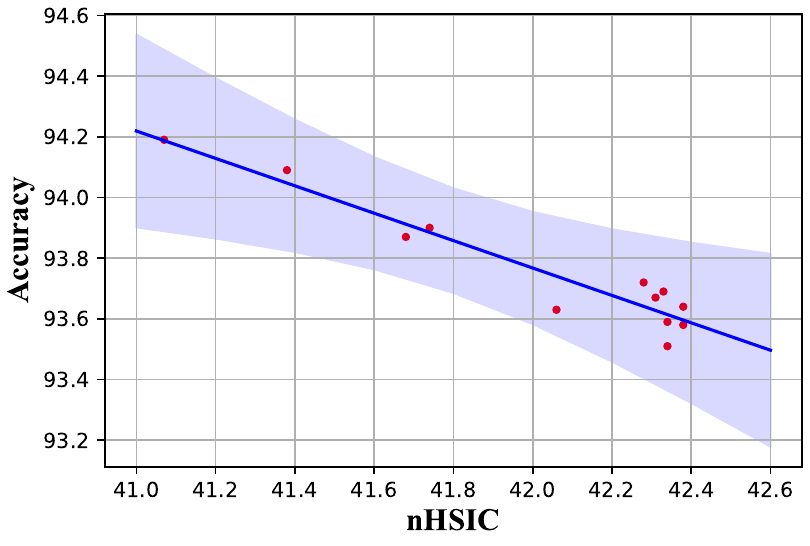}
\caption{The relationship between mutual information (MI) and accuracy on the validation set. The x-axis measures the nHSIC of the sampled feature, and the y-axis measures the accuracies of the randomly sampled architectures. The correlation between MI and accuracy is $-0.955872$. In other words, nHSIC is negatively correlated to the finetuned accuracy and thus can be used as an accurate predictor in model compression.}
\label{fig:nhsic_acc}
\end{figure}

\section{The Debate on the Information Bottleneck}
ITPruner is based on the theory of the information bottleneck in deep learning \cite{shwartz2017opening}, which is challenged by \textit{Saxe et al.} \cite{saxe2019information}. Specifically, \textit{Saxe et al.} \cite{saxe2019information} argue that the representation compression phase empirically demonstrated in \textit{Shwartz-Ziv et al.} \cite{shwartz2017opening} only appears when double-sided saturating nonlinearities like tanh and sigmoid are deployed. In this case, the claimed causality between compression and generalization was also questioned, which is foundation of ITPruner. 

To clear such debate, we first propose Theorem 1 to mathematically demonstrate that reducing the mutual information between input and the hidden representations indeed leads to a smaller generalization error. Secondly, we also provide empirical evidence in Fig.~\ref{fig:nhsic_acc} to show the performance on the validation set is highly correlated with the layer-wise mutual information. 
At last, the critical argument raised in \textit{Saxe et al.} \cite{saxe2019information} may cause by the inaccurate estimation of mutual information. Recent work \cite{wang2021pac} propose a new method to estimate the mutual information and empirically identify the original experiments in \textit{Shwartz-Ziv et al.} \cite{shwartz2017opening}. For example, Fig.~\ref{fig:mi_epoch_VGG} reports the change of mutual information (MI) between input and different layers in the training process. As we can see, there is a clear boundary between the initial ﬁtting (increase of mutual information) and the compression phases (decrease of mutual information). Overall, we thus conclude that the generalized information bottleneck principle behind ITPruner is solid theoretically and empirically. 

\begin{figure} [tb]
\centering
\includegraphics[width=1.0\linewidth]{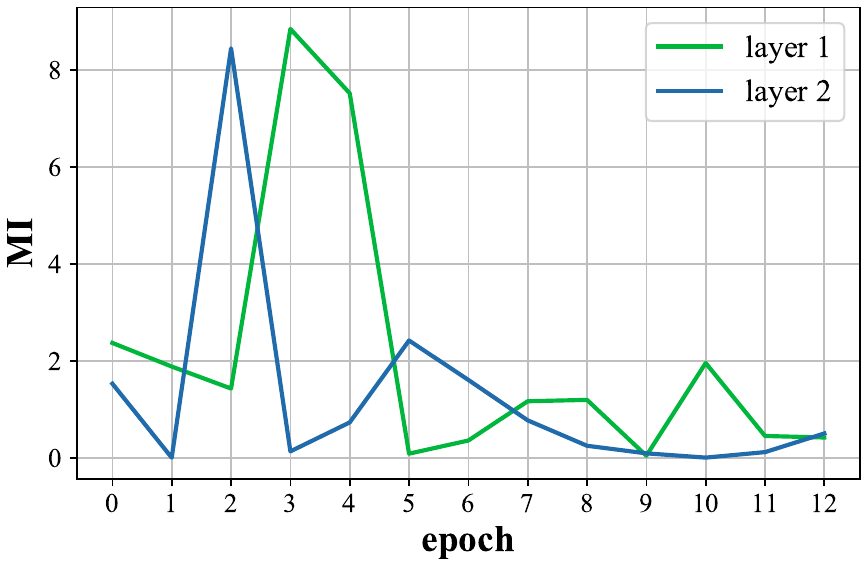}
\caption{The change of mutual information (MI) between input and different layers in training process. There is a clear boundary between the initial ﬁtting (increase of mutual information) and the compression phases (decrease of mutual information) in different layers. }
\label{fig:mi_epoch_VGG}
\end{figure}

\begin{table}[htb]
\centering
\begin{tabular}{ccc}
\midrule
\textbf{Model} & \textbf{Method} & \textbf{Acc$\%$} \\
\midrule
\multirow{4}{*}{ResNet-$20$}    & baseline & $89.65$ \\
& +Weight Decay (WD) & $92.13$ \\
& +Mutual Information (MI) & $92.01$ \\
& +MI+WD & $92.15$ \\
\midrule
\multirow{4}{*}{Invertable-ResNet} & baseline & $92.28$ \\
& +Weight Decay (WD) & $94.38$ \\
& +Mutual Information (MI) & $94.27$ \\
& +MI+WD & $94.43$ \\
\midrule
\end{tabular}
\caption{The effect of different regularization terms on ResNet-20 and Invertable-ResNet. ``baseline'' indicates that the model is trained without weight-decay.}
\label{tab:iresnet_wd_R3Q5}
\end{table}

\begin{table*}[tb]
\renewcommand\arraystretch{1.3}
\begin{center}
\begin{tabular}{cccc}
\toprule
\multicolumn{2}{c}{{\textbf{Method}}} & {\textbf{Metric}} & \textbf{Top-1 Acc}($\%$) \\
 \midrule
\multirow{5}{*}{ITPruner} & +Snip~\cite{lee2018snip}  & $\left| \frac{\partial L}{\partial W}\odot W \right|$ & $93.93$ \\
 & +Fisher~\cite{turner2019blockswap} & $S_X(X)=\left(\frac{\partial L}{\partial X} X\right)^2, S_n=\sum_{i=1}^M S_X\left(X_i\right)$ & $93.86$\\
 & +Synflow~\cite{tanaka2020pruning} & $\frac{\partial L}{\partial W} \odot W$ & $93.71$\\
  & +Grad\_Norm~\cite{abdelfattah2021zero} & $\left\|\frac{\partial L}{\partial W}\right\|_2$ & $93.94$\\
  & \textbf{+L1}~\cite{li2016pruning} & $\left\| W \right\|_1$ & $94.00$\\ 
 \midrule
\multirow{5}{*}{Uniform} & +Snip~\cite{lee2018snip} & $\left| \frac{\partial L}{\partial W}\odot W \right|$ & $92.67$\\
 & +Fisher~\cite{turner2019blockswap} & $S_X(X)=\left(\frac{\partial L}{\partial X} X\right)^2, S_n=\sum_{i=1}^M S_X\left(X_i\right)$ & $92.79$\\
 & +Synflow~\cite{tanaka2020pruning} & $\frac{\partial L}{\partial W} \odot W$ & $92.57$\\
  & +Grad\_Norm~\cite{abdelfattah2021zero} & $\left\|\frac{\partial L}{\partial W}\right\|_2$ & $92.81$\\
  & \textbf{+L1}~\cite{li2016pruning} & $\left\| W \right\|_1$ & $92.92$\\
 \bottomrule
\end{tabular}
\end{center}
\caption{Top-$1$ accuracy of compressed VGG using different local metrics and architectures on CIFAR-10.}
\label{tab:combine_other_prune_method}
\end{table*}

\section{The Debate on the Invertable-ResNet}
\textcolor{black}{In the work of \emph{Gomez et al.} \cite{gomez2017reversible} and \emph{Jacobsen et al.} \cite{jacobsen2018revnet}, they propose reversible networks, which is capable of constructing hidden representation that encode the input. Here we argue that Invertable-ResNet \cite{gomez2017reversible} is not an experimental counterexample for our Theorem \ref{theorem:compression_bottle}. }

\textcolor{black}{
Specifically, Invertable-ResNet \cite{gomez2017reversible} actually need to optimize layer-wise mutual information mentioned in Theorem \ref{theorem:Mi_and_wd} for better generalization. Firstly, in their released source code\footnote{\url{https://github.com/renmengye/revnet-public/blob/master/resnet/configs/cifar_configs.py\#L28}}, weight decay is necessary for a better generalization. We also conduct an ablation study to validate the effectiveness of weight decay. As shown in Tab.~\ref{tab:iresnet_wd_R3Q5}, both ResNet-20 and Invertable-ResNet show a significant performance promotion when the regularization term is adopted in the training. Meanwhile, as theoretically demonstrated in Theorem \ref{theorem:Mi_and_wd}, the adopted weight decay regularization is highly correlated to mutual information. Moreover, directly replacing wight decay with mutual information still shows a performance promotion, which is reported in Tab.~\ref{tab:iresnet_wd_R3Q5}. We thus conclude that optimizing mutual information in Invertable-ResNet \cite{gomez2017reversible} is necessary. }

\textcolor{black}{
There are two types of hidden representations exist in Invertable-ResNet \cite{gomez2017reversible} simultaneously. One is highly related to the input $X$ and uncorrelated to the output $Y$. The other one is the opposite, which is related to $Y$ and uncorrelated to $X$. To validate our argument, we conduct a new experiment that extracts features in Invertable-ResNet and reports the mutual information with input and output. As illustrated in Fig.~\ref{fig:irevnet_MI_relation}, most of the features have a certain degree of correlation with the $X$ and $Y$. However, the two types of the feature described before do exist in Invertable-ResNet, which is marked as a darker red dot in Fig.~\ref{fig:irevnet_MI_relation}.}

\begin{figure} [htb]
\centering
\includegraphics[width=1.0\linewidth]{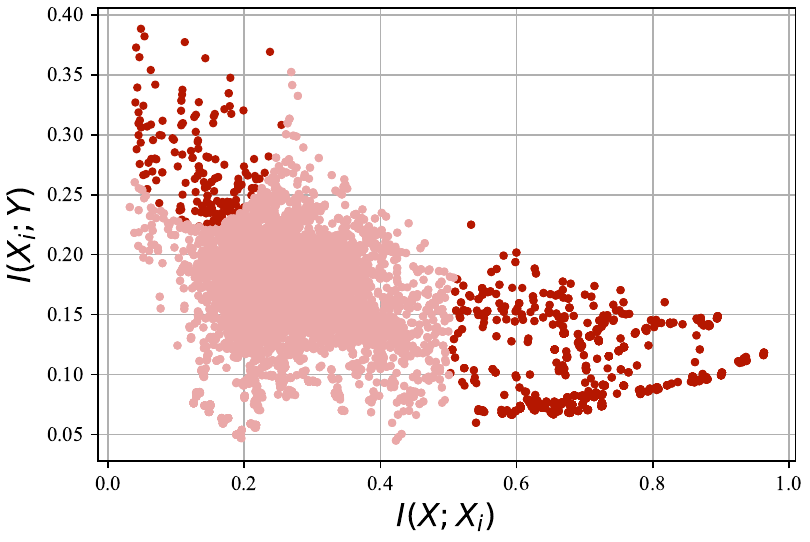}
\caption{The mutual information between the intermediate feature map and the input $X$, output $Y$, respectively. Each point represents a different channel in the feature map.}
\label{fig:irevnet_MI_relation}
\end{figure}

\setcounter{theorem}{4}

\begin{theorem}\label{theorem:Mi_and_wd}
\textcolor{black}{
Assuming that $X \sim \mathcal{N} (\boldsymbol{0},  \boldsymbol{\Sigma}_{X})$ and $Y \sim \mathcal{N} (\boldsymbol{0},  \boldsymbol{\Sigma}_{Y}), Y=W^TX$, $\min I(X;Y) \propto \min ||W||_F^2$.}
\end{theorem}
\begin{proof}
\textcolor{black}{
The definition of mutual information is $I(X;Y)=H(X)+H(Y)-H(X,Y)$. Meanwhile, the entropy of multivariate Gaussian distribution $H(X)=\frac{1}{2}\text{ln}({(2\pi e)}^D|\Sigma_X|)$ and the joint distribution $(X,Y) \sim N\left(0,\Sigma_{(X,Y)}\right)$, where 
\begin{equation}
\Sigma_{(X,Y)}=\left( \begin{array}{cc}   
    \Sigma_{X} & \Sigma_{XY}\\
    \Sigma_{YX} & \Sigma_{Y}\\
  \end{array}
\right).
\end{equation}
Therefore, the mutual information between Gaussian distributed random variables $X$ and $Y$ is represented as,
\begin{equation}
\label{MI_Gauss_Dis}
I(X;Y)=\text{ln}|\Sigma_X|+\text{ln}|\Sigma_Y|-\text{ln}|\Sigma_{(X,Y)}|.
\end{equation}
According to Everitt inequality, $|\Sigma_{(X,Y)}| \leq |\Sigma_X||\Sigma_Y|$, this inequality holds if and only if $\Sigma_{YX}=\Sigma_{XY^T}=X^TY$ is a zero matrix. 
In other words, 
\begin{equation}
\label{correlated_minimize}
\begin{aligned}
&\min I(X;Y) \\
 \propto& \min ~||\boldsymbol{X}^T\boldsymbol{Y}||_F^2 \\
\propto &\min ~||\boldsymbol{X}^T\boldsymbol{W}^T\boldsymbol{X}||_F^2.
\end{aligned}
\end{equation}
For simplicity, we take an orthogonalized $\boldsymbol{X}$. According to the unitary invariance of Frobenius norm, we have, 
\begin{equation}
\begin{aligned}
\label{correlated_proof}
 &\min~||\boldsymbol{X}^T\boldsymbol{W}^T\boldsymbol{X}||_F^2 \\
=&\min~||\boldsymbol{W}||_F^2.
\end{aligned}
\end{equation}}

\end{proof}

\section{Ablation Study}
We conduct ablation study to evaluate our calim that selecting weights in a layer need to be pruned does not help as much as finding a better $\alpha$. As we can see in Tab.~\ref{tab:combine_other_prune_method}, ITPruner always shows a significant improvement compared to their original methods. Meanwhile, integrating different pruning methods show a very similar performance. We thus conclude that selecting which weights in a layer need to be pruned does not help as much as finding a better $\alpha$, which is also well verified in the previous survey \cite{blalock2020state}.

\section{Data Availability Statement}
The data used in this study are sourced from three publicly available datasets: CIFAR-10~\cite{lecun1998gradient}, ImageNet~\cite{russakovsky2015imagenet}, and PASCAL VOC2012~\cite{everingham2010pascal}. These datasets are widely used in the research community for various machine learning and computer vision tasks, including image classification and object detection.

All datasets are available from their respective sources:

- CIFAR-10 can be accessed at \url{https://www.cs.toronto.edu/~kriz/cifar.html}.
- ImageNet can be accessed at \url{http://www.image-net.org/}.
- PASCAL VOC2012 can be accessed at \url{http://host.robots.ox.ac.uk/pascal/VOC/voc2012/}.

These datasets support the findings of this study and are available for public use. Researchers interested in using these datasets can access them through the provided URLs. There are no restrictions on the availability of these data, allowing the scientific community to freely build upon the findings of this study and advance the state-of-the-art in machine learning and computer vision.

\end{appendices}


\bibliographystyle{sn-mathphys}
\bibliography{sn-article}

\end{document}